\documentclass{article}

\usepackage{times}
\usepackage{graphicx} 
\usepackage{subfigure}

\usepackage{natbib}

\usepackage{algorithm}
\usepackage{algorithmic}

\usepackage{hyperref}

\usepackage[accepted]{icml2014}

\usepackage{amsmath}
\usepackage{amsfonts}
\usepackage{amssymb}
\usepackage{graphicx}
\usepackage{amsthm}

\newtheorem{claim}{Claim}
\newtheorem{remark}{Remark}

\newcommand{\Var}{\mathbb{V}}

\renewcommand{\v}[2]{\ensuremath{v_{#2}^{(#1)}}}

\newcommand{\E}{\ensuremath{\mathbb{E}}}

\newcommand{\hj}{\ensuremath{\widehat{j}}}
\newcommand{\bbw}{\ensuremath{\bar{\bw}}}
\newcommand{\bM}{\ensuremath{\bar{M}}}
\newcommand{\bbb}{\ensuremath{\bar{\bb}}}
\newcommand{\bT}{\ensuremath{\bar{T}}}
\newcommand{\boldeta}{\ensuremath{{\boldsymbol{\eta}}}}

\newcommand{\field}[1]{\mathbb{#1}}

\newcommand{\delete}{\mbox{\sc delete}}
\newcommand{\connected}{\mbox{\sc Is-connected}}
\newcommand{\ci}{\mbox{\sc ClusterIndices}}

\newcommand{\bb}{\boldsymbol{b}}

\newcommand{\bz}{\boldsymbol{z}}
\newcommand{\bx}{\boldsymbol{x}}
\newcommand{\ba}{\boldsymbol{a}}
\newcommand{\bw}{\boldsymbol{w}}
\newcommand{\bu}{\boldsymbol{u}}

\newcommand{\bv}{\boldsymbol{v}}

\newcommand{\bzero}{\boldsymbol{0}}

\newcommand{\R}{\field{R}}

\newcommand{\aod}{A_{\otimes}}

\renewcommand{\tilde}{\widetilde}

\renewcommand{\Pr}{\mathbb{P}}

\renewcommand{\v}[2]{\ensuremath{\bv_{#2}^{(#1)}}}

\newcommand{\ignore}[1]{}

\DeclareMathOperator*{\argmax}{argmax}

\newcommand{\CB}{\mbox{\sc cb}}
\newcommand{\sCB}{\widetilde \CB}
\newcommand{\TCB}{\mbox{\sc tcb}}
\newcommand{\sTCB}{\widetilde \TCB}

\newtheorem{theorem}{Theorem}
\newtheorem{lemma}{Lemma}

\icmltitlerunning{Online Clustering of Bandits}

\begin{document}

\twocolumn[
\icmltitle{Online Clustering of Bandits}


\icmlauthor{Claudio Gentile}{claudio.gentile@uninsubria.it}
\icmladdress{DiSTA, University of Insubria, Italy}
\icmlauthor{Shuai Li}{shuaili.sli@gmail.com}
\icmladdress{DiSTA, University of Insubria, Italy}
\icmlauthor{Giovanni Zappella}{zappella@amazon.com}
\icmladdress{Amazon Development Center Germany, Germany \\ (Work done when the author was PhD student at Univeristy of Milan)}


\vskip 0.1in
]

\begin{abstract}
We introduce a novel algorithmic approach to content recommendation based
on adaptive clustering of exploration-exploitation (``bandit") strategies.
We provide a sharp regret analysis of this algorithm in a standard stochastic
noise setting, demonstrate its scalability properties, and
prove its effectiveness on a number of artificial and real-world datasets.
Our experiments show a significant increase in prediction performance
over state-of-the-art methods for bandit problems.
%
\end{abstract}

\vspace{-0.25in}

\section{Introduction}
\vspace{-0.05in}
Presenting personalized content to users is nowdays a crucial functionality for many online
recommendation services. Due to the ever-changing set of available options, these services
have to exhibit strong adaptation capabilities when trying to match
users' preferences.
Coarsely speaking, the underlying systems repeatedly learn a mapping between available content
and users, the mapping being based on {\em context} information (that is, sets of features) which is
typically extracted from both users and contents.
The need to focus on content that raises the users' interest,
combined with the need of exploring new content so as to globally improve users' experience,
generates a well-known exploration-exploitation dilemma, which is commonly formalized as
a multi-armed bandit problem~(e.g., \cite{lr85,ACF01,audibert:hal-00711069,CaronKLB12}).
In particular, the contextual bandit methods~(e.g., \cite{Aue02,lz07,li2010contextual,chu2011contextual,
bogers2010movie,abbasi2011improved,cg11,ko11,salso11,yhg12,dkc13}, and references therein)
have rapidly become a reference algorithmic technique for implementing adaptive recommender systems.

\vspace{-0.05in}
Within the above scenarios, the widespread adoption of online social networks, where users are
engaged in technology-mediated social interactions (making product endorsement and word-of-mouth
advertising a common practice), raises further challenges and opportunities to content recommendation
systems: On one hand, because of the mutual influence among friends, acquaintances, business partners,
etc., users having strong ties are more likely to exhibit similar interests, and therefore similar
behavior. On the other hand, the nature and scale of such interactions calls for adaptive
algorithmic solutions which are also computationally affordable.

\vspace{-0.05in}
Incorporating social components into bandit algorithms can lead to a dramatic increase in the quality
of recommendations. For instance, we may want to serve content to a group of users by taking advantage of
an underlying network of social relationships among them. These social relationships can either be
explicitly encoded in a graph, where adjacent nodes/users are deemed similar to one another, or implicitly
contained in the data, and given as the outcome of an inference process that recognizes similarities
across users based on their past behavior.
Examples of the first approach are the recent works~\cite{bes13,dkmc13,cgz13}, where
a social network structure over the users is assumed to be given that reflects actual interest similarities
among users -- see also \cite{cb13, valko2014spectral} for recent usage of social information to tackle the so-called
``cold-start" problem. Examples of the second approach are the more traditional
collaborative-filtering (e.g.,~\cite{skr99}), content-based filtering, and hybrid approaches (e.g.~\cite{b05}).

Both approaches have important drawbacks hindering their practical deployment. One obvious drawback
of the ``explicit network" approach is that the social network information may be misleading
(see, e.g., the experimental evidence reported by~\cite{dkmc13}), or simply unavailable.
Moreover, even in the case when this information is indeed available and useful,
the algorithmic strategies to implement the needed feedback sharing mechanisms might lead to severe
scaling issues~\cite{cgz13}, especially when the number of targeted users is large.
A standard drawback of the ``implicit network" approach of traditional recommender systems
is that in many practically relevant scenarios (e.g., web-based), content universe and popularity often
undergo dramatic changes, making these approaches difficult to apply.

\vspace{-0.05in}
In such settings, most notably in the relevant case
when the involved users are many, it is often possible to identify
a few subgroups or communities within which users share similar interests~\cite{rlkr06,bwdh12},
thereby greatly facilitating
the targeting of users by means of {\em group} recommendations. Hence the system need not learn
a different model for each user of the service, but just a single model for each group.

\vspace{-0.05in}
In this paper, we carry out\footnote
{
Due to space limitations, we postpone the discussion of related work to the supplementary
material.
}
a theoretical and experimental investigation of adaptive clustering algorithms
for linear (contextual) bandits under the assumption that we have to serve content to a set of $n$ users
organized into $m << n$ groups (or {\em clusters}) such that users within each group tend to provide
similar feedback to content recommendations. We give a $O(\sqrt{T})$ regret analysis holding in a standard
stochastically linear setting for payoffs where, importantly, the hidden constants in the big-oh
depend on $m$, rather than $n$, as well as on the geometry of the user models within the different
clusters.
The main idea of our algorithm is to use confidence balls of the users' models to both estimate user similarity,
and to share feedback across (deemed similar) users. The algorithm adaptively interpolates between
the case when we have a single instance of a contextual bandit algorithm making the same predictions for all
users and the case when we have $n$-many instances providing fully personalized recommendations.
We show that our algorithm can be implemented efficiently (the large $n$ scenario being of special concern here)
by means of off-the-shelf data-structures relying on random graphs. Finally, we test our algorithm
on medium-size synthetic and real-world datasets, often reporting a significant increase in prediction performance
over known state-of-the-art methods for bandit problems.

\vspace{-0.07in}
\section{Learning Model}\label{s:model}
\vspace{-0.05in}
We assume the user behavior similarity 
is encoded as
an {\em unknown} clustering of the users. Specifically, let
$V = \{1,\ldots, n\}$ represent the set of $n$ users. Then $V$ can be partitioned
into a small number $m$ of clusters $V_1, V_2, \ldots, V_m$, with $m << n$, such that
users lying in the same cluster share similar behavior and users lying in different
clusters have 
different behavior. The actual partition of $V$
(including the number of clusters $m$) and the common user behavior within each
cluster are unknown to the learner, and have to be inferred on the fly.

\vspace{-0.05in}
Learning proceeds in a sequential fashion: At each round $t=1,2,\dots$,
the learner receives a user index $i_t \in V$ together with a set of
context vectors $C_{i_t} = \{\bx_{t,1}, \bx_{t,2},\ldots, \bx_{t,c_t}\} \subseteq \R^d$.
The learner then selects some ${\bar \bx_t} =  \bx_{t,k_t} \in C_{i_t}$ to recommend to user $i_t$,
and observes some payoff $a_t \in \R$, which is a function of both $i_t$ and the recommended
${\bar \bx_t}$.
The following assumptions are made on how index $i_t$, set $C_{i_t}$, and payoff $a_t$
are generated in round $t$. Index $i_t$ represents the user to be served by the system,
and we assume $i_t$ is selected uniformly at random\footnote
{
Any other distribution that insures a positive probability of visiting each node of $V$ would suffice here.
}
from $V$. Once $i_t$ is selected,
the number of context vectors $c_t$ in $C_{i_t}$
is generated arbitrarily as a function of past indices $i_1, \ldots, i_{t-1}$, payoffs $a_1, \ldots, a_{t-1}$,
and sets $C_{i_1}, \ldots, C_{i_{t-1}}$, as well as the current index $i_t$.
Then
the sequence $\bx_{t,1}, \bx_{t,2},\ldots, \bx_{t,c_t}$ of context vectors
within $C_{i_t}$
%
is generated i.i.d. (conditioned on $i_t, c_t$ and all past indices $i_1, \ldots, i_{t-1}$,
payoffs $a_1, \ldots, a_{t-1}$, and sets $C_{i_1}, \ldots, C_{i_{t-1}}$)
from a random process on the surface of the unit sphere,
whose process matrix $\E[XX^\top]$ is full rank, with minimal eigenvalue $\lambda > 0$.
Further assumptions on the process matrix $\E[XX^\top]$ are made later on.
Finally, payoffs are generated by noisy versions of unknown linear functions
of the context vectors. That is, we assume each cluster $V_j$, $j = 1, \ldots, m$,
hosts an unknown parameter vector $\bu_j \in \R^d$ which is common to each user $i \in V_j$.
Then the payoff value
$a_i(\bx)$ associated with user $i$ and context vector $\bx \in \R^d$ is given by the random variable
\[
a_i(\bx) = \bu_{j(i)}^\top\bx + \epsilon_{j(i)}(\bx)~,
\]
where $j(i) \in \{1, 2, \ldots, m\}$ is the index of the cluster that node $i$ belongs to, and
$\epsilon_{j(i)}(\bx)$ is a conditionally zero-mean and bounded variance noise term.
Specifically, denoting by $\E_t[\,\cdot\,]$ the conditional expectation\
\(
\E\bigl[\,\cdot\,\big|\, (i_1, C_{i_1}, a_1), \ldots, (i_{t-1}, C_{i_{t-1}}, a_{t-1}), i_t\,\bigr],
\)
we assume that for any fixed $j \in \{1,\ldots,m\}$ and $\bx \in \R^d$, the variable
$\epsilon_j(\bx)$ is such that
$\E_t[\epsilon_{j}(\bx)|\,\bx\,] = 0$ and
$\Var_t\bigl[\epsilon_{j}(\bx)|\,\bx\,\bigr] \leq \sigma^2$, where $\Var_t[\,\cdot\,]$ is
a shorthand for the conditional variance
$
\Var\bigl[\,\cdot\,\big|\, (i_1, C_{i_1}, a_1), \ldots, (i_{t-1}, C_{i_{t-1}}, a_{t-1}),i_t\,\bigr]
$
of the variable at argument.
So we clearly have $\E_t[a_i(\bx)|\,\bx\,] = \bu_{j(i)}^\top\bx$ and
$\Var_t\bigl[a_i(\bx)|\,\bx\,\bigr] \leq \sigma^2$.
Therefore, $\bu_{j(i)}^\top\bx$ is the expected payoff observed at user $i$
for context vector $\bx$. 
In the special case when the noise
$\epsilon_{j(i)}(\bx)$ is a bounded random variable taking values in the range
$[-1,1]$, this implies $\sigma^2 \leq 1$. We will make throughout the assumption that
$a_i(\bx) \in [-1,1]$ for all $i \in V$ and $\bx$. Notice that this implies
$-1 \leq \bu_{j(i)}^\top\bx \leq 1$ for all $i \in V$ and $\bx$.
Finally, we assume well-separatedness among the clusters, in that\
\(
||\bu_{j}-\bu_{j'}|| \geq \gamma > 0\ {\mbox{for all $j \neq j'$}}.
\)
\ We define the regret $r_t$ of the learner at time $t$ as
%
\vspace{-0.05in}
\[
r_t = \left(\max_{\bx \in C_{i_t} }\, \bu_{j(i_t)}^\top\bx\right) - \bu_{j(i_t)}^\top{\bar \bx_{t}}~.
\]
We are aimed at bounding with high probability (over the variables $i_t$,
$\bx_{t,k}$, $k = 1,\ldots, c_t$,
and the noise variables $\epsilon_{j(i_t)}$) the cumulative regret
\(
\sum_{t=1}^T r_t~.
\)
The kind of regret bound we would like to obtain (we call it the {\em reference} bound) is one where
the clustering structure of $V$ (i.e., the partition of $V$ into $V_1, \ldots, V_m$) is known to the algorithm
ahead of time, and we simply view each one of the $m$ clusters as an independent bandit problem.
In this case, a standard contextual bandit analysis \cite{Aue02,chu2011contextual,abbasi2011improved}
shows that, as $T$ grows large, the cumulative
regret $\sum_{t=1}^T r_t$ can be bounded with high probability as\footnote
{
The ${\tilde O}$-notation hides logarithmic factors.
}
\vspace{-0.05in}
\begin{center}
\(
\sum_{t=1}^T r_t
=
{\tilde O}\Bigl(\sum_{j = 1}^m \left(\sigma\,d + ||\bu_j||\,\sqrt{d}\right)\,\sqrt{T}\Bigl)~.
\)
\end{center}
%
For simplicity, we shall assume that $||\bu_j|| = 1$ for all $j = 1, \ldots, m$. Now,
a more careful analysis exploiting our assumption about the randomness of $i_t$ (see the supplementary
material) reveals that one can replace
the $\sqrt{T}$ term contributed by each bandit $j$ by a term of the form
$\sqrt{T}\,\left(\frac{1}{m}+\sqrt{\frac{|V_j|}{n}}\right)$, so that
under our assumptions the reference bound becomes
\vspace{-0.02in}
\begin{equation}\label{e:referencebound}
\sum_{t=1}^T r_t
=
{\tilde O}\Biggl(\left(\sigma\,d + \sqrt{d}\right)\sqrt{T}\Bigl(1+\sum_{j=1}^m \sqrt{\frac{|V_j|}{n}}\Bigl)\Biggl)~.
\end{equation}
%
Observe the dependence of this bound on the size of clusters $V_j$. The worst-case scenario is when
we have $m$ clusters of the same size $\frac{n}{m}$, resulting in the bound
\vspace{-0.08in}
\begin{center}
\(
\sum_{t=1}^T r_t
=
{\tilde O}\left(\left(\sigma\,d + \sqrt{d}\right)\,\sqrt{m\,T}\right)~.
\)
\end{center}
%
At the other extreme lies the easy case when we have a single big cluster and many small ones.
For instance, $|V_1| = n-m+1$, and $|V_2| = |V_3| = \ldots |V_m| = 1$, for $m << n$, gives
\vspace{-0.05in}
\begin{center}
\(
\sum_{t=1}^T r_t
=
{\tilde O}\Bigl(\left(\sigma\,d + \sqrt{d}\right)\,\sqrt{T}\,\Bigl(1+\frac{m}{\sqrt{n}}\Bigl)\Bigl)~.
\)
\end{center}
%
A relevant geometric parameter of the set of $\bu_j$ is the {\em sum of distances}
$SD(\bu_j)$ of a given vector $\bu_j$
w.r.t. the set of vectors $\bu_1, \ldots, \bu_m$,
which we define as\
\(
SD(\bu_j) = \sum_{\ell = 1}^m ||\bu_j - \bu_{\ell}||.\
\)
%
If it is known that $SD(\bu_j)$ is small for all $j$, one can modify the abovementioned
independent bandit algorithm, by letting the bandits share signals,
as is done, e.g., in \cite{cgz13}. This allows one to exploit the vicinity of the $\bu_j$ vectors,
and roughly replace
\(
1+\sum_{j=1}^m \sqrt{\frac{|V_j|}{n}}
\)
in (\ref{e:referencebound}) by a quantity also depending on the mutual distances
$||\bu_j - \bu_{j'}||$ among cluster vectors.
However, this improvement is obtained at the cost of a substantial increase of running time \cite{cgz13}.
In our analysis (Theorem \ref{t:regret} in Section \ref{s:alg}), we would like to leverage both
the geometry of the clusters, as encoded by vectors $\bu_j$, and the
relative size $|V_j|$ of the clusters, with no prior knowledge of $m$ (or $\gamma$), and
without too much extra computational burden.

\vspace{-0.1in}
\section{The Algorithm}\label{s:alg}
\vspace{-0.05in}
Our algorithm, called Cluster of Bandits (CLUB), is described in Figure \ref{alg:club}.
In order to describe the algorithm
we find it convenient to re-parameterize the problem described in Section \ref{s:model},
and introduce $n$ parameter vectors $\bu_1, \bu_2, \ldots, \bu_n$, one per node, where nodes
within the same cluster $V_j$ share the same vector. An illustrative example is given in
Figure \ref{f:1}.

\vspace{-0.05in}
The algorithm maintains at time $t$ an estimate $\bw_{i,t}$ for
vector $\bu_i$ associated with user $i \in V$. Vectors $\bw_{i,t}$ are updated based
on the payoff signals, similar to a standard linear bandit algorithm
(e.g.,~\cite{chu2011contextual}) operating on the context vectors contained in
$C_{i_t}$. Every user $i$ in $V$ hosts a linear bandit algorithm
like the one described in \cite{cgz13}.
One can see that the
prototype vector $\bw_{i,t}$ is the result of a standard linear least-squares
approximation to the corresponding unknown parameter vector $\bu_i$.
In particular, $\bw_{i,t-1}$ is defined through the inverse correlation matrix $M^{-1}_{i,t-1}$,
and the additively-updated vector $\bb_{i,t-1}$.
Matrices $M_{i,t}$ are initialized to the $d\times d$ identity matrix,
and vectors $\bb_{i,t}$ are initialized to the $d$-dimensional zero vector.
In addition, the algorithm maintains at time $t$ an undirected
graph $G_t = (V,E_t)$ whose nodes are precisely the users in $V$. The algorithm starts off from the complete
graph, and progressively erases edges based on the evolution of vectors $\bw_{i,t}$. The graph is intended
to encode the current partition of $V$ by means of the {\em connected components} of $G_t$.
We denote by $\hat V_{1,t}, \hat V_{2,t}, \ldots, \hat V_{m_t,t}$ the partition of $V$ induced by the connected
components of $G_t$. Initially, we have $m_1 = 1$ and
$\hat V_{1,1} = V$. The clusters $\hat V_{1,1}, \hat V_{2,t}, \ldots, \hat V_{m_t,t}$
(henceforth called the {\em current} clusters)
are indeed meant to estimate the underlying true partition $V_1, V_2, \ldots, V_{m}$, henceforth called
the {\em underlying} or {\em true} clusters.

%
\begin{figure}[t!]
\begin{center}
\begin{algorithmic}
\small
\STATE \textbf{Input}: Exploration parameter $\alpha > 0$; edge deletion parameter $\alpha_2 > 0$
\STATE \textbf{Init}:
\vspace{-0.1in}
\begin{itemize}
\item $\bb_{i,0} = \bzero \in \R^d$ and $M_{i,0} = I \in \R^{d\times d}$,\ \ $i = 1, \ldots n$;
\vspace{-0.1in}
\item Clusters ${\hat V_{1,1}} = V$, number of clusters $m_1 = 1$;
\vspace{-0.1in}
\item Graph $G_1 = (V,E_1)$, $G_1$ is connected over $V$.
\end{itemize}
\vspace{-0.1in}
\FOR{$t =1,2,\dots,T$}
\STATE Set $\bw_{i,t-1} = M_{i,t-1}^{-1}\bb_{i,t-1}$, \quad $i = 1, \ldots, n$;
\STATE Receive $i_t \in V$, and get context
       $
       C_{i_t} = \{\bx_{t,1},\ldots,\bx_{t,c_t}\};
       $
\STATE Determine $\hj_t \in \{1, \ldots, m_t\}$ such that $i_t \in {\hat V_{\hj_t,t}}$, and set
\vspace{-0.1in}
\begin{align*}
   \bM_{{\hj_t},t-1}  &= I+\sum_{i \in {\hat V_{\hj_t,t}}} (M_{i,t-1}-I),\\
   \bbb_{{\hj_t},t-1} &= \sum_{i \in {\hat V_{\hj_t,t}}} \bb_{i,t-1},\\
   \bbw_{{\hj_t},t-1} &= \bM_{{\hj_t},t-1}^{-1}\bbb_{{\hj_t},t-1}~;
\end{align*}
\vspace{-0.4in}
\STATE \[
       {\mbox{Set\ \ }}
        k_t = \argmax_{k = 1, \ldots, c_t} \left({\bbw_{\hj_t,t-1}}^\top\bx_{t,k}  + \CB_{\hj_t,t-1}(\bx_{t,k})\right),
       \]
\vspace{-0.2in}
       \begin{align*}
       \CB_{j,t-1}(\bx) &= \alpha\,\sqrt{\bx^\top \bM_{j,t-1}^{-1} \bx\,\log(t+1)},\\
       \bM_{j,t-1} &= I+\sum_{i \in {\hat V_{j,t}}} (M_{i,t-1}-I)\,,\quad j = 1, \ldots, m_t~.
       \end{align*}
\vspace{-0.1in}
\STATE Observe payoff $a_t \in [-1,1]$;
\STATE Update weights:
       \begin{itemize}
       \vspace{-0.1in}
       \item $M_{i_t,t} = M_{i_t,t-1} + {\bar \bx_{t}}{\bar \bx_{t}}^\top$,
       \vspace{-0.1in}
       \item $\bb_{i_t,t} = \bb_{i_t,t-1} + a_t {\bar \bx_t}$,
       \vspace{-0.1in}
       \item Set $M_{i,t} =  M_{i,t-1},\ \bb_{i,t} = \bb_{i,t-1}$ for all $i \neq i_t$~;
       \end{itemize}
\STATE Update clusters:
       \begin{itemize}
       \vspace{-0.1in}
       \item Delete from $E_t$ all $(i_t,\ell)$ such that
       \[
             ||\bw_{i_t,t-1} - \bw_{\ell,t-1}|| > \sCB_{i_t,t-1} + \sCB_{\ell,t-1}~,
       \]
       \vspace{-0.3in}
       \begin{align*}
        \sCB_{i,t-1} &= \alpha_2\,\sqrt{\frac{1+\log(1+T_{i,t-1})}{1+T_{i,t-1}}},\\
        T_{i,t-1}    &= |\{s \leq t-1\,:\, i_s = i \}|,\qquad i \in V;
       \end{align*}
       \vspace{-0.3in}
       \item Let $E_{t+1}$ be the resulting set of edges, set $G_{t+1} = (V,E_{t+1})$,
        and compute associated clusters
        $\hat V_{1,t+1}, \hat V_{2,t+1}, \ldots, \hat V_{m_{t+1},t+1}$~.
       \end{itemize}
\ENDFOR
\end{algorithmic}
\vspace{-0.1in}
\caption{\label{alg:club}
Pseudocode of the CLUB algorithm.
The confidence functions $\CB_{j,t-1}$ and $\sCB_{i,t-1}$ are simplified versions of
their ``theoretical" counterparts $\TCB_{j,t-1}$ and $\sTCB_{i,t-1}$, defined later on.
The factors $\alpha$ and $\alpha_2$ are used here as tunable parameters that bridge the simplified
versions to the theoretical ones.
\vspace{-0.25in}
}
\end{center}
\end{figure}
%
%
%
\vspace{-0.05in}
At each time $t=1,2,\dots$, the algorithm receives the index $i_t$ of the user to serve, and the
associated context vectors $\bx_{t,1}, \ldots, \bx_{t,c_t}$
(the set $C_{i_t}$), and must select one among them. In doing so, the algorithm first determines
which cluster (among $\hat V_{1,1}, \hat V_{2,t}, \ldots, \hat V_{m_t,t}$)
node $i_t$ belongs to, call this cluster $\hat V_{\hj_t,t}$,
then builds the aggregate weight vector $\bbw_{\hj_t,t-1}$ by taking prior ${\bar \bx_s}$, $s < t$,
such that $i_s \in \hat V_{\hj_t,t}$, and computing the least squares approximation as if all nodes
$i \in \hat V_{\hj_t,t}$ have been collapsed into one.
%
%
It is weight vector $\bbw_{{\hj_t},t-1}$ that the algorithm uses to select $k_t$.
In particular,
\vspace{-0.04in}
\[
       k_t = \argmax_{k = 1, \ldots, c_t} \left({\bbw_{\hj_t,t-1}}^\top\bx_{t,k}  + \CB_{\hj_t,t-1}(\bx_{t,k})\right)~.
\]
The quantity $\CB_{\hj_t,t-1}(\bx)$ is a version of the upper confidence bound in the approximation of
$\bbw_{\hj_t,t-1}$ to a suitable combination
of vectors $\bu_i$, $i \in {\hat V_{\hj_t,t}}$ -- see the supplementary material for details.

\vspace{-0.05in}
Once this selection is done and the associated payoff $a_t$ is observed, the algorithm uses the selected vector
${\bar \bx_t}$ for updating $M_{i_t,t-1}$ to $M_{i_t,t}$ via a rank-one adjustment,
and for turning vector $\bb_{i_t,t-1}$ to $\bb_{i_t,t}$
via an additive update whose learning rate is precisely $a_t$.
Notice that the update is only performed at node $i_t$, since for all other $i \neq i_t$ we have
$\bw_{i,t} = \bw_{i,t-1}$. However, this update at $i_t$ will also implicitly update the
aggregate weight vector $\bbw_{\hj_{t+1},t}$ associated with cluster ${\hat V_{\hj_{t+1},t+1}}$ that node $i_t$
will happen to belong to in the next round.
%
Finally, the cluster structure is possibly modified. At this point
CLUB compares, for all existing edges $(i_t,\ell) \in E_{t}$, the distance
\(
||\bw_{i_t,t-1}- \bw_{\ell,t-1}||
\)
between
vectors $\bw_{i_t,t-1}$ and $\bw_{\ell,t-1}$
to the quantity
\(
\sCB_{i_t,t-1} + \sCB_{\ell,t-1}~.
\)
If the above distance is significantly large (and $\bw_{i_t,t-1}$ and $\bw_{\ell,t-1}$ are good approximations to
the respective underlying vectors $\bu_{i_t}$ and $\bu_{\ell}$),
%
%
then this is a good indication that
$\bu_{i_t} \neq \bu_{\ell}$ (i.e., that node $i_t$ and node $\ell$ cannot belong to the same true cluster), so that
edge $(i_t,\ell)$ gets deleted.
The new graph $G_{t+1}$, and the induced partitioning clusters $\hat V_{1,t+1}, \hat V_{2,t+1}, \ldots, \hat V_{m_{t+1},t+1}$,
are then computed, and a new round begins.

\vspace{-0.1in}
\subsection{Implementation}\label{ss:implementation}
\vspace{-0.05in}
In implementing the algorithm in Figure \ref{alg:club}, the reader should bear in mind that we are
expecting $n$ (the number of users) to be quite large, $d$ (the number of features of each item) to be relatively small,
and $m$ (the number of true clusters) to be very small compared to $n$.
With this in mind, the algorithm can be implemented by storing a least-squares estimator $\bw_{i,t-1}$
at each node $i \in V$, an aggregate least squares estimator $\bbw_{\hj_t,t-1}$ for each current cluster
$\hj_t \in \{1,\ldots, m_t\}$, and an extra data-structure which is able to perform decremental dynamic connectivity.
Fast implementations of such data-structures are those studied by~\cite{Tho97,kkm13}
(see also the research thread referenced therein).
One can show (see the supplementary material) that
in $T$ rounds we have an overall (expected) running time
\vspace{-0.05in}
\begin{align}
O\Bigl(&T\,\Bigl(d^2 + \frac{|E_1|}{n}\,d\Bigl) + m\,(n\,d^2 + d^3) + |E_1|\notag\\
&+ \min\{n^2, |E_1|\,\log n\} + \sqrt{n\,|E_1|}\,\log^{2.5} n\Bigl)~.\label{e:runningtime}
\end{align}
%
Notice that the above is $n\cdot{\mbox{poly}}(\log n)$,
if so is $|E_1|$. In addition, if $T$ is large compared to $n$ and $d$, the average running time per round
becomes $O(d^2 + d\cdot{\mbox{poly}}(\log n))$.
As for memory requirements, this implementation
takes $O(n\,d^2 + m\,d^2 + |E_1|) = O(n\,d^2+|E_1|)$. Again, this is $n\cdot{\mbox{poly}}(\log n)$
if so is $|E_1|$.

\vspace{-0.05in}
\subsection{Regret Analysis}\label{s:analysis}
\vspace{-0.05in}
%
%
Our analysis relies on the high probability analysis contained in~\cite{abbasi2011improved}
(Theorems~1 and 2 therein).
The analysis (Theorem \ref{t:regret} below) is carried out in the case when
the initial graph $G_1$ is the complete graph. However, if the true clusters are sufficiently large,
then we can show (see Remark \ref{r:sparse}) that a formal statement can be made even if we
start off from sparser random graphs, with substantial time and memory savings.

\vspace{-0.05in}
The analysis actually refers to a version of the algorithm where the confidence bound
functions $\CB_{j,t-1}(\cdot)$ and $\sCB_{i,t-1}$ in Figure \ref{alg:club}
are replaced by their ``theoretical" counterparts
$\TCB_{j,t-1}(\cdot)$, and $\sTCB_{i,t-1}$, respectively,\footnote
{
Notice that, in all our notations, index $i$ always ranges over nodes, while index $j$ always ranges over clusters.
Accordingly, the quantities $\sCB_{i,t}$ and $\sTCB_{i,t}$ are always associates with node $i \in V$, while
the quantities $\CB_{j,t-1}(\cdot)$ and $\TCB_{j,t-1}(\cdot)$ are always associates with clusters $j \in \{1,\ldots,m_t\}$.
}
which are defined as follows. Set for brevity
\vspace{-0.05in}
\[
A_{\lambda}(T,\delta) \hspace{-0.03in} 
= \hspace{-0.03in} \left(\frac{\lambda\,T}{4}\hspace{-0.03in}-\hspace{-0.03in}8\log\Bigl(\frac{T+3}{\delta}\Bigl)
- 2\sqrt{T\log\Bigl(\frac{T+3}{\delta }\Bigl)} \right)_+
\]
where $(x)_+ = \max\{x,0\}$, $x \in \R$.
Then, for $j = 1, \ldots, m_t$,
\vspace{-0.1in}
\begin{equation}\label{e:tcb}
\TCB_{j,t-1}(\bx)
=
\sqrt{\bx^\top \bM_{j,t-1}^{-1}\bx}\Biggl(\sigma\sqrt{2\log \frac{|\bM_{j,t-1}|}{\delta/2}} + 1\Biggl),
\end{equation}
%
being $|\cdot|$ the determinant of the matrix at argument, and, for $i \in V$,
\vspace{-0.1in}
\begin{equation}\label{e:stcb}
\sTCB_{i,t-1}
=  \frac{\sigma\sqrt{2d\,\log t + 2\log (2/\delta)} + 1}
         {\sqrt{1+ A_{\lambda}(T_{i,t-1}, \delta/(2nd))}}\,.
\end{equation}
%
%
%
Recall the difference between {\em true} clusters $V_1, \ldots, V_m$ and
{\em current} clusters $\hat V_{1,t}, \ldots, \hat V_{m_t,t}$ maintained by the algorithm at time $t$. Consistent
with this difference, we let $G = (V,E)$ be the true underlying graph, made up of the $m$ disjoint cliques over
the sets of nodes  $V_1, \ldots, V_m \subseteq V$, and $G_t = (V,E_t)$ be the one kept by the algorithm --
see again Figure \ref{f:1} for an illustration of how the algorithm works.
\begin{figure}[t!]
\begin{picture}(7,152)(7,152)
\scalebox{0.4}{\includegraphics{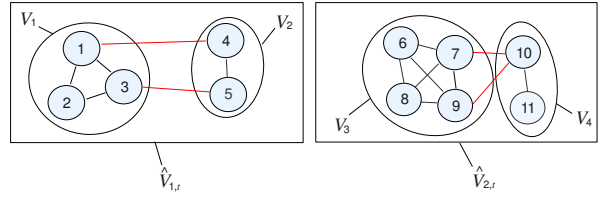}}
\end{picture}
\vspace{-1.2in}
\caption{A true underlying graph $G = (V,E)$ made up of $n = |V| = 11$ nodes, and $m = 4$ true clusters
$V_1 = \{1,2,3\}$, $V_2 = \{4,5\}$, $V_3 = \{6,7,8,9\}$, and $V_4 = \{10,11\}$.
There are $m_t = 2$ current clusters
$\hat V_{1,t} $ and $\hat V_{2,t} $. The black edges are the ones contained in $E$,
while the red edges are those contained in $E_t\setminus E$. The two current clusters also correspond to the
two connected components of graph $G_t = (V,E_t)$.
Since aggregate vectors $\bbw_{j,t}$ are build based on current cluster membership, if for instance, $i_t = 3$, then $\hj_t = 1$, so
$\bM_{1,t-1} = I + \sum_{i = 1}^5 (M_{i,t-1}-I)$, $\bbb_{1,t-1} = \sum_{i = 1}^5 \bb_{i,t-1}$, and
$\bbw_{1,t-1} = \bM_{1,t-1}^{-1}\bbb_{1,t-1}$.
%
%
\label{f:1}
}
\vspace{-0.15in}
\end{figure}
The following is the main theoretical result of this paper,\footnote{
The proof is provided in the supplementary material.
}
where additional conditions are needed on the process $X$ generating the context vectors.
\begin{theorem}\label{t:regret}
Let the CLUB algorithm of Figure~\ref{alg:club} be run on the initial complete graph $G_1= (V,E_1)$,
whose nodes $V = \{1,\ldots,n\}$ can be partitioned into $m$ clusters $V_1, \ldots, V_m$ where,
for each $j = 1, \ldots, m$,
nodes within cluster $V_j$ host the same vector $\bu_j$, with $||\bu_j|| = 1$ for $j = 1, \ldots, m$, and
$||\bu_j-\bu_{j'}|| \geq \gamma > 0$ for any $j \neq j'$.
Denote by $v_j = |V_j|$ the cardinality of cluster $V_j$.
Let the $\CB_{j,t}(\cdot)$ function in Figure~\ref{alg:club} be replaced by the
$\TCB_{j,t}(\cdot)$ function defined in (\ref{e:tcb}), and $\sCB_{i,t}$ be replaced by $\sTCB_{i,t}$
defined in (\ref{e:stcb}). In both $\TCB_{j,t}$ and $\sTCB_{i,t}$, let $\delta$ therein be replaced by $\delta/10.5$.
Let, at each round $t$, context vectors $C_{i_t} = \{\bx_{t,1}, \ldots, \bx_{t,c_t}\}$ being generated
i.i.d. (conditioned on $i_t, c_t$ and all past indices $i_1, \ldots, i_{t-1}$,
payoffs $a_1, \ldots, a_{t-1}$, and sets $C_{i_1}, \ldots, C_{i_{t-1}}$)
from a random process $X$ such that $||X|| = 1$, $\E[XX^\top]$ is full rank, with minimal eigenvalue
$\lambda > 0$. Moreover, for any fixed unit vector $\bz \in \R^d$, let the random variable
\(
(\bz^\top X)^2
\)
be (conditionally) sub-Gaussian with variance parameter\
\(
\nu^2 = \Var_t\bigl[(\bz^\top X)^2\,|\,c_t\,\bigr]
\leq \frac{\lambda^2}{8\log (4c)},\
\)
with $c_t \leq c$ for all $t$.
Then with probability at least $1-\delta$ the cumulative regret satisfies
\vspace{-0.08in}
\begin{align}
\sum_{t=1}^T r_t
&\hspace{-0.03in}=\hspace{-0.03in} {\tilde O}\Biggl(\hspace{-0.03in} (\sigma\sqrt{d}+1)\,\sqrt{m}\,\Biggl(\frac{n}{\lambda^2}
+\sqrt{T}\Bigl(1+\sum_{j=1}^m \sqrt{\frac{v_j}{\lambda\,n}} \Bigl)\hspace{-0.03in}\Biggl)\notag\\
&\qquad\ \  + \left(\frac{n}{\lambda^2} + \frac{n\,\sigma^2\,d\,}{\lambda\gamma^2} \right)\,\E[SD(\bu_{i_t})] + m \Biggl)\notag\\
&\hspace{-0.03in}=\hspace{-0.03in} {\tilde O}\Biggl( (\sigma\sqrt{d}+1)\,\sqrt{m\,T}\,\Bigl(1+\sum_{j=1}^m \sqrt{\frac{v_j}{\lambda\,n}} \Bigl) \Biggl)~,
\label{e:asymptotic_bound}
\end{align}
as $T$ grows large.
In the above, the ${\tilde O}$-notation hides $\log (1/\delta)$, $\log m$, $\log n$, and $\log T$ factors.
\end{theorem}
%
%
\vspace{-0.04in}
\begin{remark}
A close look at the cumulative regret bound presented in Theorem \ref{t:regret} reveals that this bound is
made up of three main terms: The first term is of the form
\vspace{-0.04in}
\[
(\sigma\sqrt{dm}+\sqrt{m})\,\frac{n}{\lambda^2} + m~.
\]
This term is constant with
$T$, and essentially accounts for the transient regime due to the
convergence of the minimal eigenvalues of $\bM_{j,t}$ and $M_{i,t}$ to the corresponding minimal eigenvalue $\lambda$
of $\E[XX^\top]$.
The second term is of the form
\vspace{-0.07in}
\[
\Bigl(\frac{n}{\lambda^2} + \frac{n\,\sigma^2\,d\,}{\lambda\gamma^2} \Bigl)\,\E[SD(\bu_{i_t})]~.
\]
This term is again constant with $T$, but it depends through $\E[SD(\bu_{i_t})]$ on the geometric properties
of the set of $\bu_j$ as well as on the way such $\bu_j$ interact with the cluster sizes $v_j$.
Specifically,
%
\vspace{-0.04in}
\begin{center}
\(
\E[SD(\bu_{i_t})]
= \sum_{j=1}^m \frac{v_j}{n}\,\sum_{j'=1}^m ||\bu_j-\bu_{j'}||~.
\)
\end{center}
%
Hence this term is small if, say, among the $m$ clusters, a few of them together cover almost all nodes in $V$
(this is a typical situation in practice) and,
in addition, the corresponding $\bu_j$ are close to one another. This term accounts for the hardness of learning
the true underlying clustering through edge pruning.
We also have an inverse dependence on $\gamma^2$, which is likely due to an artifact of our analysis.
Recall that $\gamma$ is not known to our algorithm.
Finally, the third term is the one characterizing the asymptotic behavior of our algorithm as $T \rightarrow \infty$,
its form being just (\ref{e:asymptotic_bound}). It is instructive to compare this term to the reference bound
(\ref{e:referencebound}) obtained by assuming prior knowledge of the cluster structure.
Broadly speaking, (\ref{e:asymptotic_bound}) has an extra $\sqrt{m}$ factor,\footnote
{
This extra factor could be eliminated at the cost of having a higher second term in the bound, which does not
leverage the geometry of the set of $\bu_j$.
}
and replaces a factor $\sqrt{d}$
in (\ref{e:referencebound}) by the larger
factor $\sqrt{\frac{1}{\lambda}}$.
\end{remark}
\vspace{-0.05in}
\begin{remark}\label{r:emptygraph}
The reader should observe that a similar algorithm as CLUB can be designed that starts off from the empty graph
instead, and progressively draws edges (thereby merging connected components and associated aggregate vectors)
as soon as two nodes host individual vectors $\bw_{i,t}$ which are close enough to one another.
This would have the advantage to lean on even faster data-structures for maintaining disjoint sets
(e.g., \cite{CLR90}[Ch. 22]),
but has also the significant drawback of requiring prior knowledge of the separation parameter $\gamma$.
In fact, it would not be possible to connect two previously unconnected nodes without knowing something about
this parameter.
A regret analysis similar to the one in Theorem \ref{t:regret} exists, though our current understanding is that
the cumulative regret would depend linearly on $\sqrt{n}$ instead of $\sqrt{m}$.
Intuitively, this algorithm is biased towards a large number of true clusters, rather
than a small number.
\end{remark}
\vspace{-0.05in}
\begin{remark}\label{r:data_dependent}
A data-dependent variant of the CLUB algorithm can be designed and analyzed which relies on
data-dependent clusterability assumptions of the set of users with respect to a set of context vectors.
These data-dependent assumptions allow us to work in a fixed design setting for the sequence of context vectors
$\bx_{t,k}$, and remove the sub-Gaussian and full-rank hypotheses regarding $\E[XX^\top]$.
On the other hand, they also require that the power of the adversary generating context vectors be suitably
restricted. See the supplementary material for details.
\end{remark}

\vspace{-0.05in}
\begin{remark}\label{r:sparse}
Last but not least, we would like to stress that the same analysis contained in Theorem \ref{t:regret} extends
to the case when we start off from a $p$-random Erdos-Renyi initial graph $G_1 = (V,E_1)$, where
$p$ is the independent probability that two nodes are connected by an edge in $G_1$.
Translated into our context, a classical result on random graphs due to~\cite{Kar94}
reads as follows.
\begin{lemma}
Given $V = \{1, \ldots, n\}$, let $V_1, \ldots, V_m$ be a partition of $V$, where $|V_j| \geq s$ for all $j = 1,\ldots, m$.
Let $G_1 = (V,E_1)$ be a $p$-random Erdos-Renyi graph with\
\(
p \geq \frac{12\,\log(6n^2/\delta)}{s-1}~.
\)
\ Then with probability at least $1-\delta$ (over the random draw of edges), all $m$ subgraphs
induced by true clusters $V_1, \ldots, V_m$ on $G_1$ are connected in $G_1$.
\end{lemma}
\vspace{-0.05in}
For instance, if $|V_j| = \beta\,\frac{n}{m}$, $j = 1,\ldots, m$, for some constant $\beta \in (0,1)$,
then it suffices to have\
\(
|E_1| = O\left(\frac{m\,n\,\log (n/\delta)}{\beta}\right)~.
\)
\ Under these assumptions, if the initial graph $G_1$ is such a random graph, it is easy to show that
Theorem \ref{t:regret} still holds.
As mentioned in Section \ref{ss:implementation} (Eq. (\ref{e:runningtime}) therein), the striking advantage of
beginning with a sparser connected graph than the complete graph is computational, since we need not handle
anymore a (possibly huge) data-structure having $n^2$-many items.
In our experiments, described next, we set $p = \frac{3\,\log n}{n}$, so as to be reasonably
confident that $G_1$ is (at the very least) connected.

%
%
%
\end{remark}

\vspace{-0.18in}
\section{Experiments}\label{s:exp}
\vspace{-0.07in}
We tested our algorithm on both artificial and freely available
real-world datasets against standard bandit baselines.

\vspace{-0.08in}
\subsection{Datasets}
\vspace{-0.05in}
\textbf{Artificial datasets.}
We firstly generated synthetic datasets, so as to have
a more controlled experimental setting. We tested the relative performance of the algorithms
along different axes: number of underlying clusters, balancedness of cluster sizes, and amount of payoff noise.
We set $c_t = 10$ for all $t = 1, \ldots, T$, with time horizon $T = 5,000 + 50,000$,
$d = 25$, and  $n = 500$.
For each cluster $V_j$ of users, we created a random
unit norm vector $\bu_j \in \R^d$. All $d$-dimensional context vectors $\bx_{t,k}$
%
%
have then been generated uniformly at random on the surface of the Euclidean ball.
The payoff value associated with cluster vector $\bu_j$ and context vector $\bx_{t,k}$ has been generated
by perturbing the inner product $\bu_j^\top\bx_{t,k}$ through an additive white noise term $\epsilon$ drawn
uniformly at random across the interval $[-\sigma,\sigma]$.
It is the value of $\sigma$ that determines the amount of payoff noise.
The two remaining parameters
are the number of clusters $m$ and the clusters' relative size.
We assigned to cluster $V_j$ a number of users $|V_j|$ calculated as\footnote
{
We took the integer part in this formula, and reassigned
the remaining fractionary parts of users to the first cluster.
}
$|V_j| = n\,\frac{j^{-z}}{\sum_{\ell=1}^m \ell^{-z}}$, $j = 1, \ldots, m$,
with $z \in \{0,1,2,3\}$, so that $z = 0$ corresponds to equally-sized clusters, and
$z = 3$ yields highly unbalanced cluster sizes.
Finally, the sequence of served users $i_t$ is generated uniformly at random over the $n$ users.

\vspace{-0.05in}
\textbf{LastFM \& Delicious datasets.}
These datasets are extracted from the music streaming service Last.fm and the
social bookmarking web service Delicious.
The LastFM dataset contains $n = 1,\!892$ nodes, and $17,\!632$ items (artists).
This dataset contains information about the listened artists, and we used this information to
create payoffs: if a user listened to an artist at least once the payoff is $1$, otherwise the payoff is $0$.
Delicious is a dataset with $n = 1,\!861$ users, and $69,\!226$ items (URLs).
The payoffs were created using the information about the bookmarked URLs for each user:
the payoff is $1$ if the user bookmarked the URL, otherwise the payoff is $0$.\footnote
{
Datasets and their full descriptions are available at \texttt{www.grouplens.org/node/462}.
}
These two datasets are inherently different: on Delicious, payoffs depend on users more strongly
than on LastFM, that is, there are more popular artists whom everybody listens to
than popular websites which everybody bookmarks. LastFM is a ``few hits" scenario, while
Delicious is a ``many niches" scenario, making a big difference in recommendation practice.
Preprocessing was carried out by closely following previous experimental settings, like
the one in \cite{cgz13}. In particular, we only retained the first 25 principal components of the context
vectors resulting from a tf-idf representation of the available items, so that on both datasets $d=25$.
We generated random context sets $C_{i_t}$ of size $c_t = 25$ for all $t$ by selecting index $i_t$
at random over the $n$ users, then picking $24$ vectors at random from the available items,
and one among those with nonzero payoff for user
$i_t$.\footnote
{
This is done so as to avoid a meaningless comparison: With high probability,
a purely random selection would result in payoffs equal to zero for all the context vectors in $C_{i_t}$.
}
We repeated this process $T = 5,000 +  50,000$ times for the two datasets.

\textbf{Yahoo dataset.} We extracted two datasets from the one
adopted by the ``ICML 2012 Exploration and Exploitation 3 Challenge''\footnote
{
\texttt{https://explochallenge.inria.fr/}
}
for news article recommendation.
Each user is represented by a 136-dimensional binary feature vector, and
we took this feature vector as a proxy for the identity of the user.
We operated on the first week of data. After removing ``empty" users,\footnote
{
Out of the 136 Boolean features, the first feature is always 1 throughout all records.
We call ``empty" the users whose only nonzero feature is the first feature.
}
this gave rise to a dataset of $8,362,905$ records, corresponding to $n = 713,862$ distinct users.
The overall number of distinct news items turned out to be $323$, $c_t$ changing from round to
round, with a maximum of $51$, and a median of $41$.
The news items have no features, hence they
have been represented as $d$-dimensional {\em versors}, with $d = 323$.
Payoff values $a_t$ are either 0 or 1 depending on whether
the logged web system which these data refer to has observed a positive (click) or negative (no-click)
feedback from the user in round $t$.
We then extracted the two datasets ``5k users"
and ``18k users" by filtering out users that have occurred less than $100$ times and less than $50$ times,
respectively.
Since the system's recommendation need not coincide
with the recommendation issued by the algorithms we tested, we could only retain the records on which
the two recommendations were indeed the same.
Because records are discarded on the fly, the actual number of retained records changes
across algorithms, but it is about $50,000$ for the ``5k users" version and  about $70,000$ for the
``18k users" version.

\vspace{-0.09in}
\subsection{Algorithms}
\vspace{-0.05in}
We compared CLUB with two main competitors: LinUCB-ONE and LinUCB-IND.
Both competitors are members of the LinUCB family of algorithms~\cite{Aue02,chu2011contextual,
li2010contextual,abbasi2011improved,cgz13}.
LinUCB-ONE allocates a single instance of LinUCB across all users (thereby making the same
prediction for all users), whereas LinUCB-IND (``LinUCB INDependent") allocates an independent instance of LinUCB
to each user, thereby making predictions in a fully personalised fashion.
Moreover, on the synthetic experiments,
we added two idealized baselines: a GOBLIN-like algorithm~\cite{cgz13} fed
with a Laplacian matrix encoding the true underlying graph $G$, and
a CLAIRVOYANT algorithm that knows the true clusters a priori, and runs one instance of LinUCB {\em per cluster}.
Notice that an experimental comparison to multitask-like algorithms, like GOBLIN, or to the 
idealized algorithm that knows all clusters beforehand, can only be done on the artificial datasets, 
not in the real-world case where no cluster information is available. 
On the Yahoo dataset, we tested the featureless version of the LinUCB-like algorithm in~\cite{cgz13}, which
is essentially a version of the UCB1 algorithm of~\cite{ACF01}.
The corresponding ONE and IND versions are denoted by UCB-ONE and UCB-IND, respectively.
On this dataset, we also tried a single instance of UCB-V~\cite{audibert:hal-00711069} across all users,
the winner of the abovementioned ICML Challenge.
Finally, all algorithms have also been compared
to the trivial baseline (denoted by RAN) that picks the item within $C_{i_t}$ fully at random. 

\vspace{-0.06in}
As for parameter tuning, CLUB was run with $p = \frac{3\,\log n}{n}$, so as to be reasonably confident
that the initial graph is at least connected.
In fact, after each generation of the graph, we checked for its
connectedness, and repeated the process until the graph happened to be connected.\footnote
{
Our results are averaged over 5 random initial graphs, but this randomness
turned out to be
a minor source of variance. 
}
All algorithms (but RAN) require parameter tuning: an exploration-exploitation tradeoff parameter
which is common to all algorithms (in Figure \ref{alg:club}, this is the $\alpha$ parameter),
and the edge deletion parameter $\alpha_2$ in CLUB.
On the synthetic datasets, as well as on the LastFM and Delicious datasets, we tuned these parameters
by picking the best setting (as measured by cumulative regret) after the first $t_0 = 5,000$ rounds, and
then sticked to those values for the remaining $T-t_0 = 50,000$ rounds.
It is these $50,000$ rounds that our plots refer to.
%
On the Yahoo dataset,
this optimal tuning was done within the first $t_0 = 100,000$ records, corresponding
to a number of retained records between $4,350$ and $4,450$ across different algorithms.
\vspace{-0.09in}
\subsection{Results}
\vspace{-0.05in}
Our results are summarized in\footnote{Further plots can be found in the supplementary material.}
Figures \ref{fig:artificiala}, \ref{fig:last_del}, and \ref{fig:yahoo}.
On the synthetic datasets (Figure \ref{fig:artificiala}) and
the LastFM and Delicious datasets (Figure \ref{fig:last_del}) we measured the ratio of
the cumulative regret of the algorithm
to the cumulative regret of the random predictor RAN (so that the lower the better).
On the synthetic datasets, we did so under combinations of number of clusters, payoff noise,
and cluster size balancedness.
On the Yahoo dataset (Figure \ref{fig:yahoo}), because the only available payoffs are those
associated with the
items recommended in the logs, we instead measured the Clickthrough Rate (CTR), i.e., the fraction
of times we get $a_t = 1$ out of the number of retained records so far (so the higher
the better). This experimental setting is in line with previous ones (e.g., \cite{li2010contextual})
and, by the way data have been prepared, gives rise to a reliable estimation of actual CTR behavior
under the tested experimental conditions \cite{lclw11}.

\begin{figure}[t!]
\begin{picture}(0,31)(0,31)
\begin{tabular}{l@{\hspace{-1.2pc}\vspace{6.4pc}}l}
\includegraphics[width=0.23\textwidth]{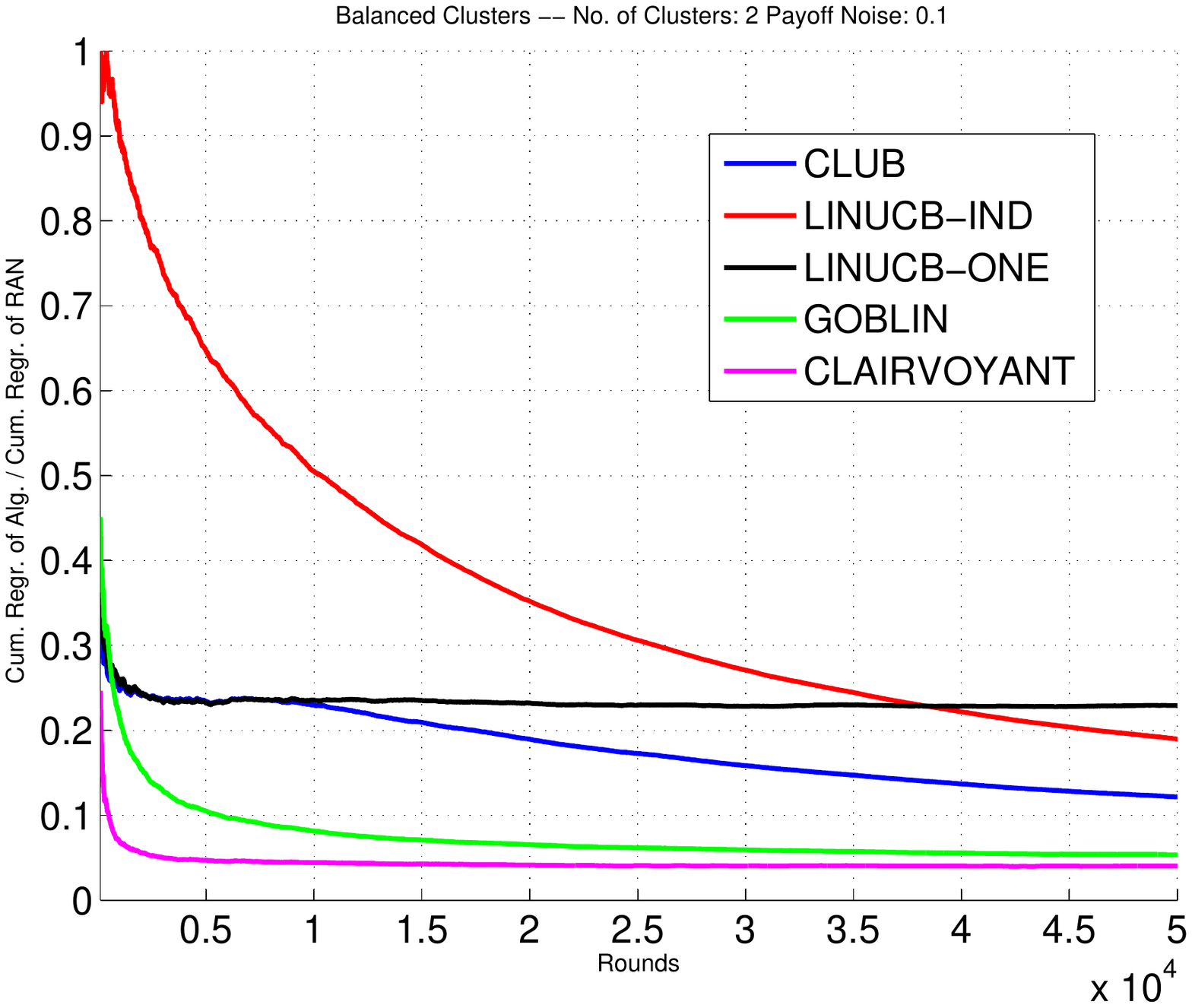}
& \includegraphics[width=0.23\textwidth]{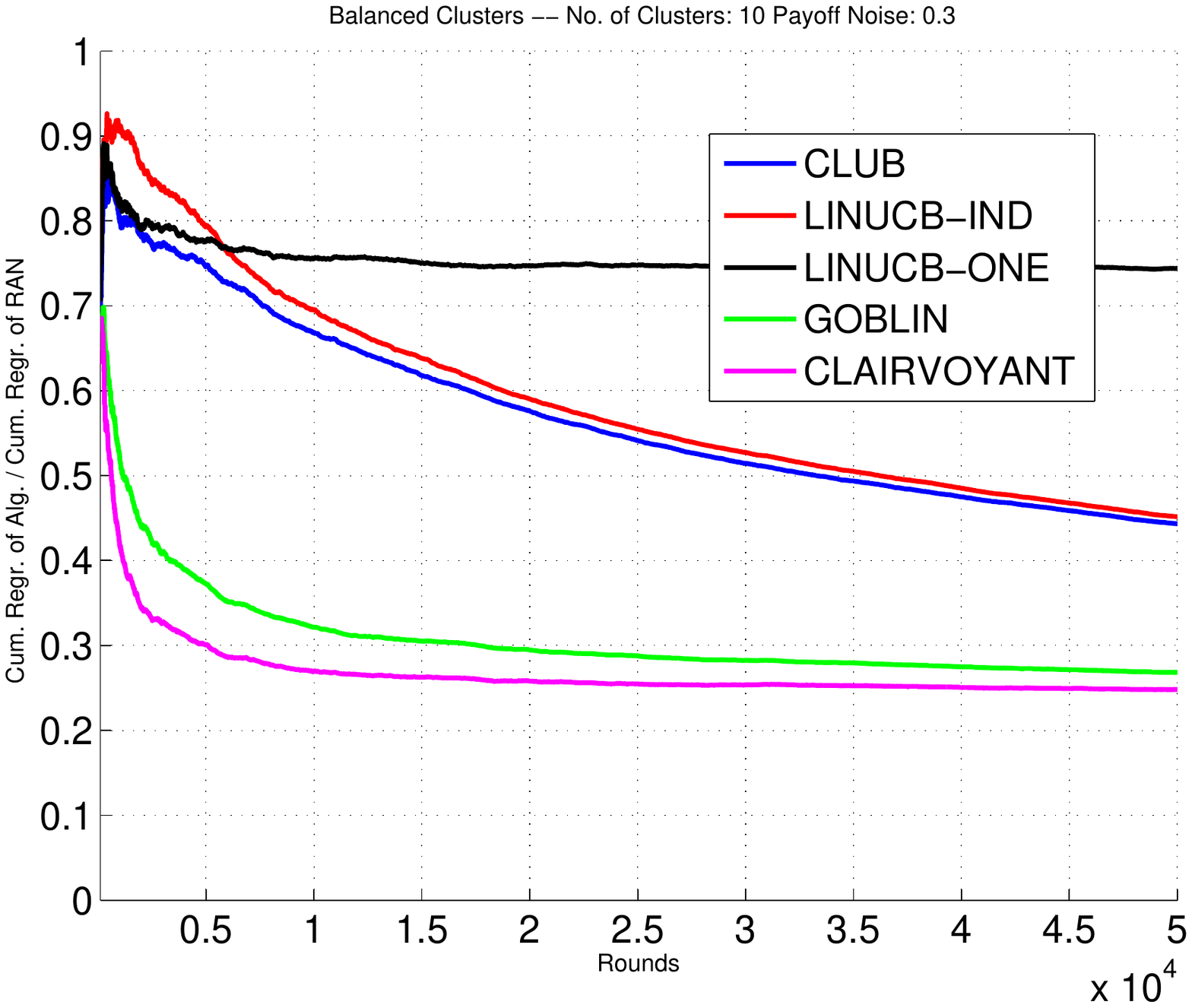}\\[-2.0in]
\includegraphics[width=0.23\textwidth]{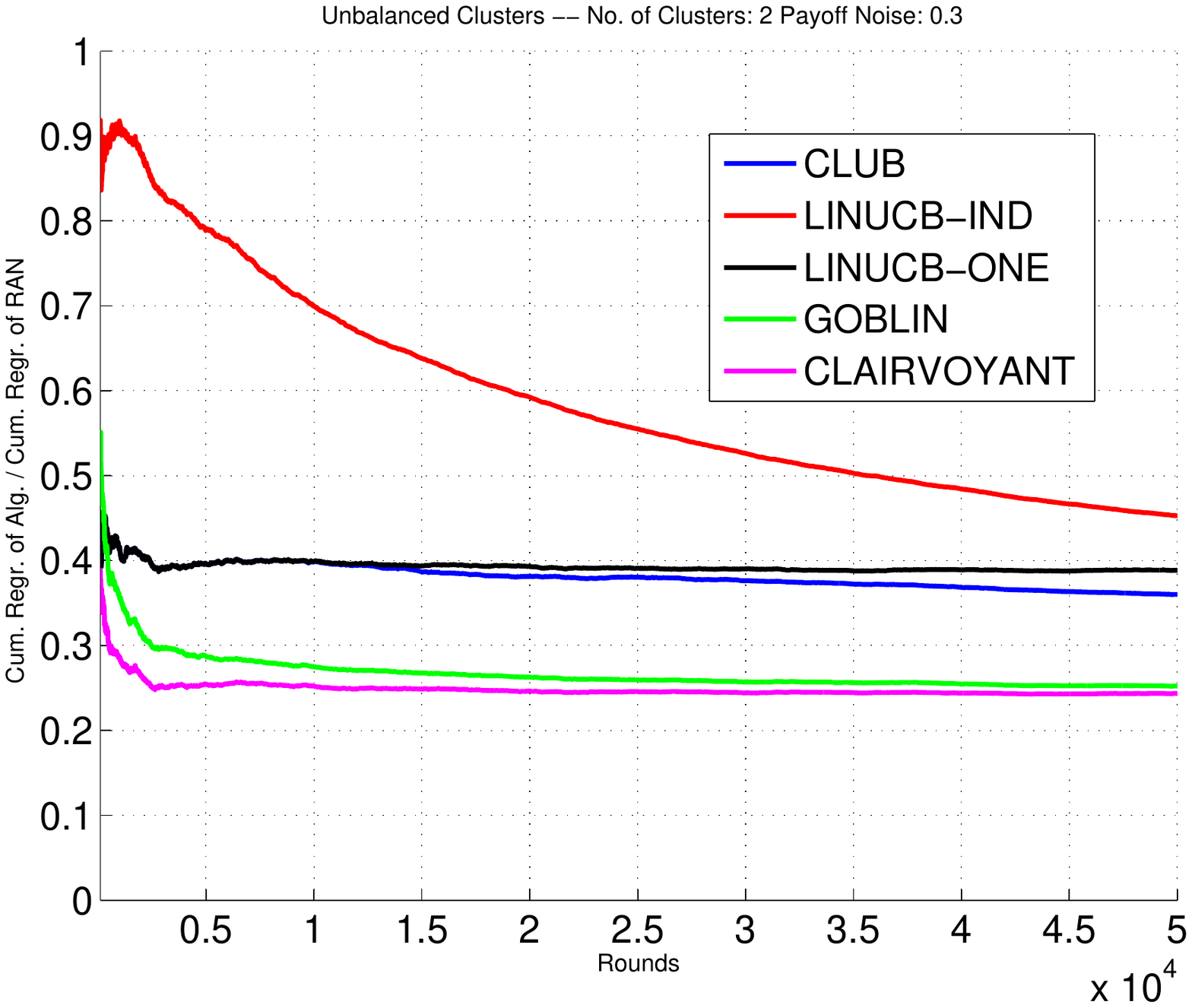}
& \includegraphics[width=0.23\textwidth]{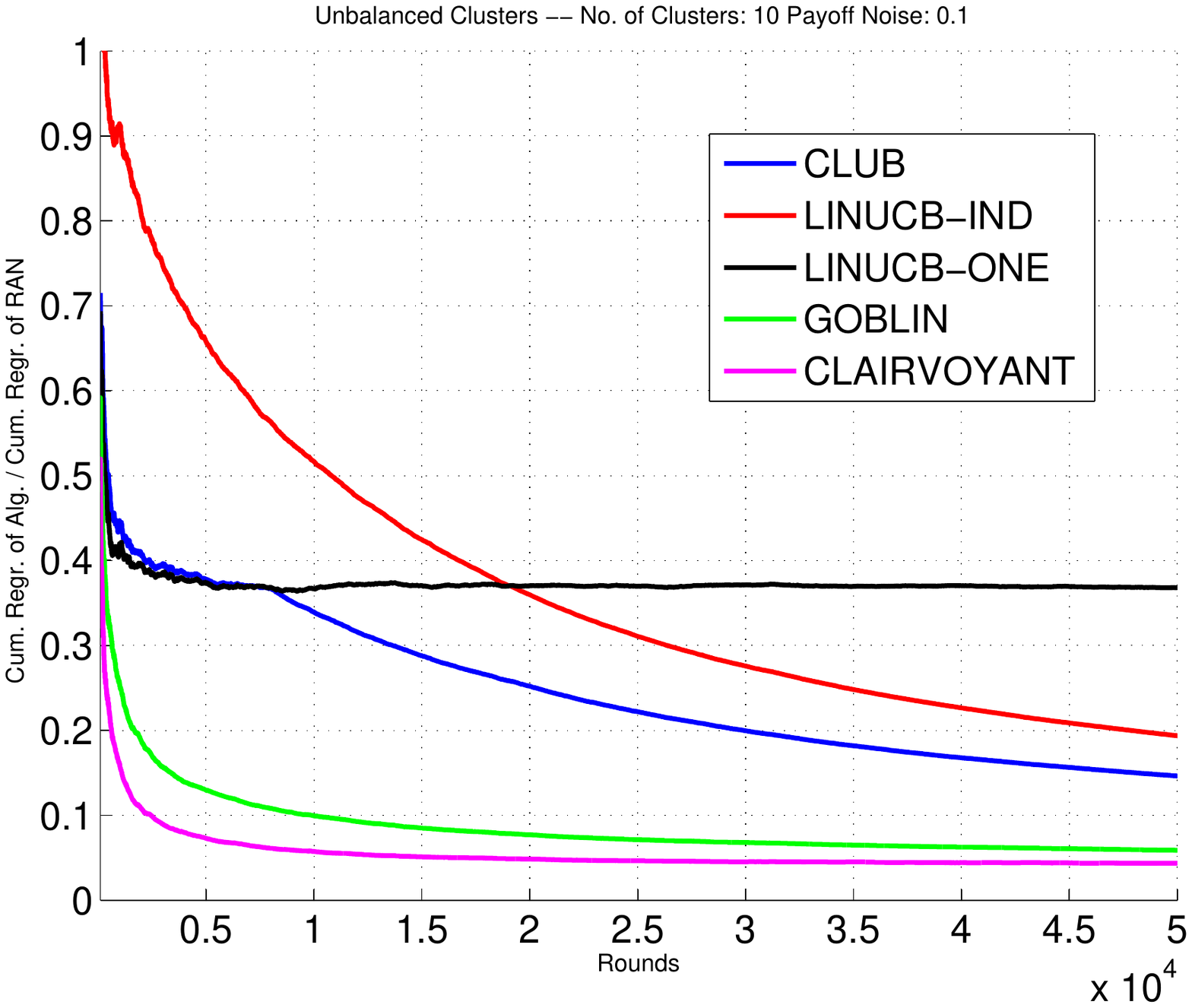}\\[-2.0in]
\end{tabular}
\end{picture}
\vspace{1.9in}
\caption{\label{fig:artificiala}Results on synthetic datasets. Each plot displays the behavior of
the ratio of the current cumulative regret of the algorithm (``Alg") to the current cumulative regret
of RAN, where ``Alg" is either ``CLUB" or ``LinUCB-IND" or ``LinUCB-ONE"
or ``GOBLIN''or ``CLAIRVOYANT''. 
In the top two plots cluster sizes are balanced ($z=0$), while in the bottom two they are unbalanced ($z=2$).
}
\vspace{-0.25in}
\end{figure}
\begin{figure}[t!]
\begin{picture}(6,22)(6,22)
\begin{tabular}{l@{\hspace{0.22pc}}l}
\includegraphics[width=0.234\textwidth]{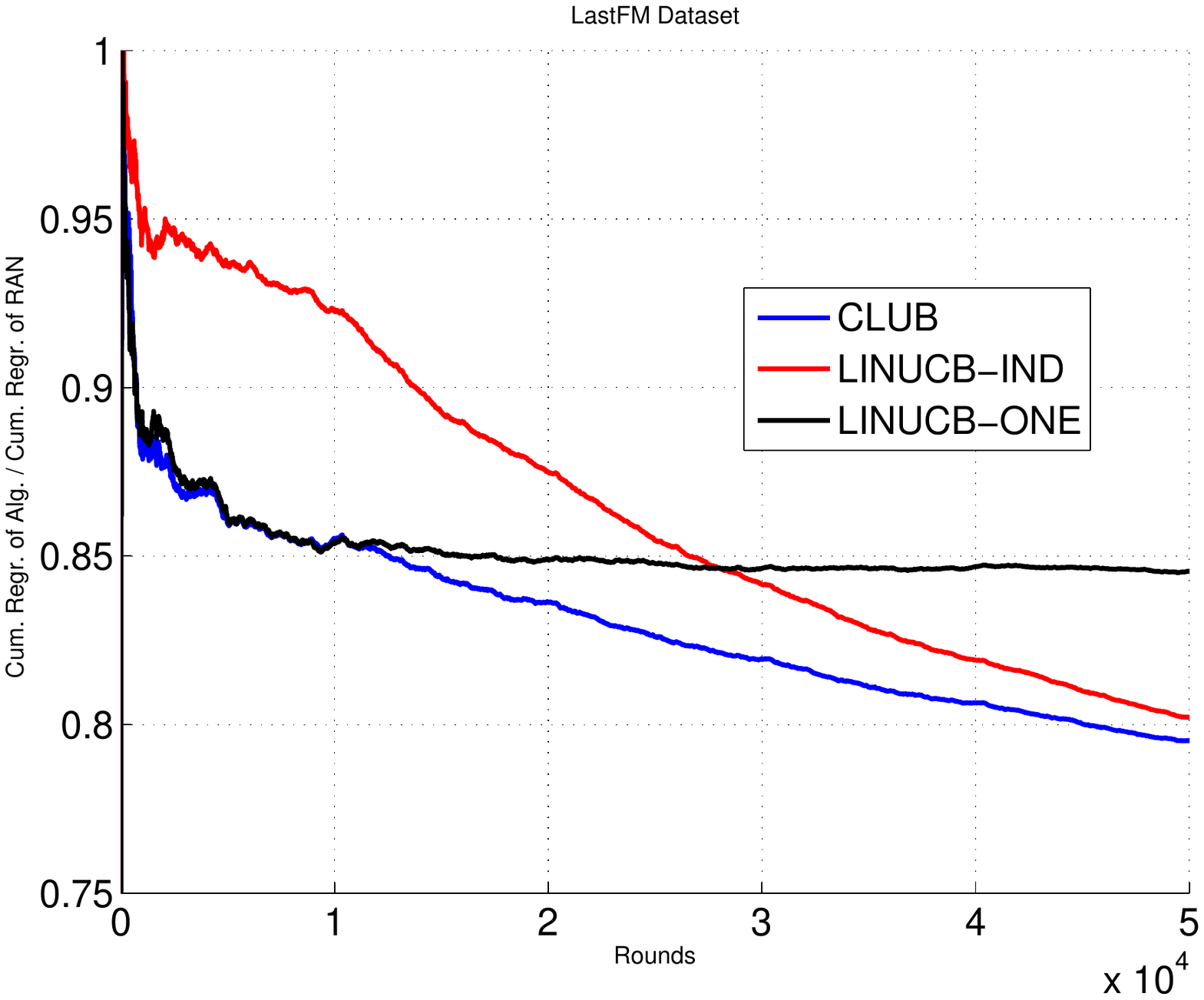}
& \includegraphics[width=0.234\textwidth]{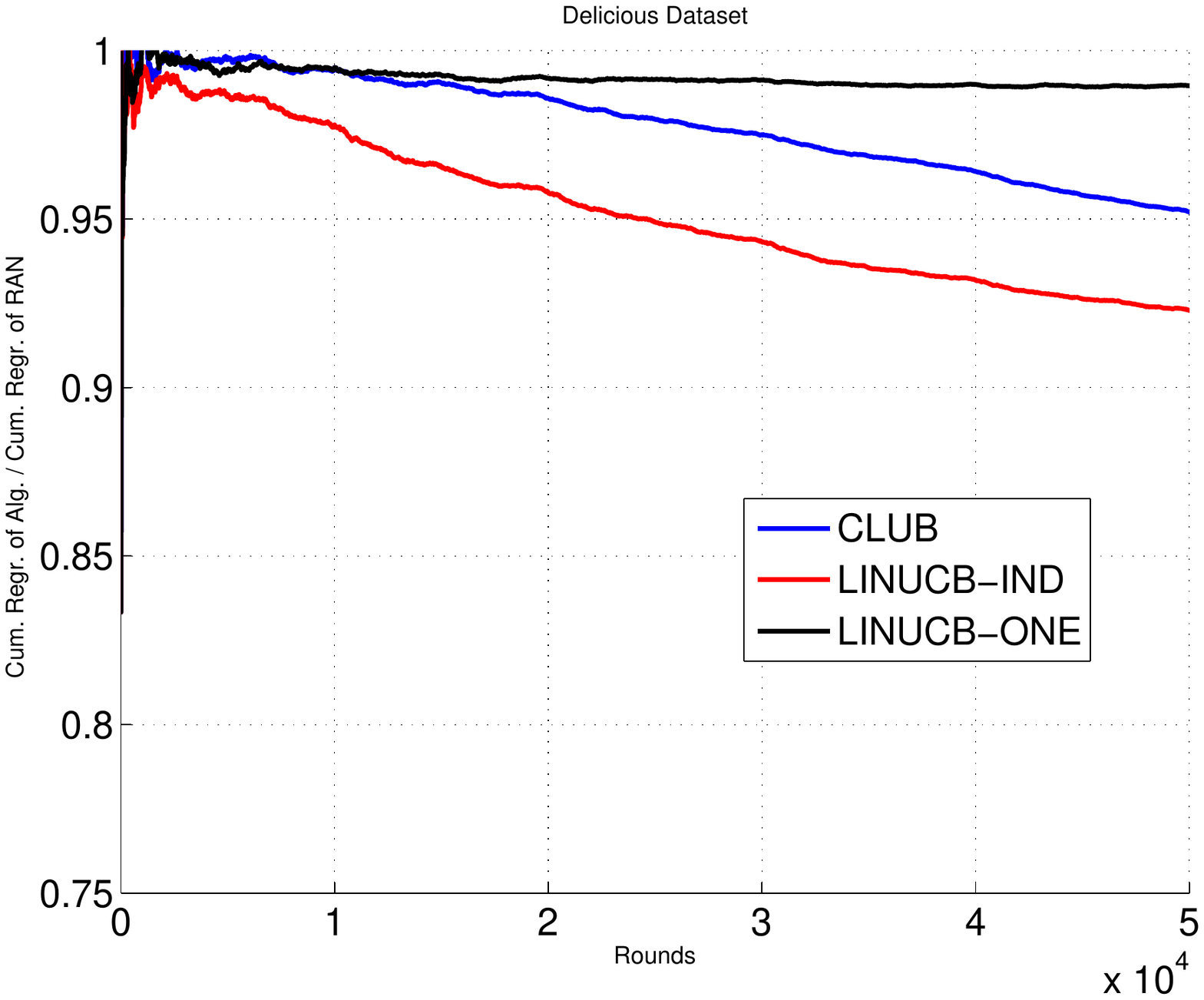}
\end{tabular}
\end{picture}
\vspace{0.85in}
\caption{\label{fig:last_del}Results on the LastFM (left) and the Delicious (right) datasets.
The two plots display the behavior of the ratio of the current cumulative regret of the algorithm
(``Alg") to the current cumulative regret of RAN, where ``Alg" is either ``CLUB" or ``LinUCB-IND"
or ``LinUCB-ONE".}
\vspace{-0.1in}
\end{figure}

\vspace{-0.05in}
Based on the experimental results, some trends can be spotted: 
On the \textbf{synthetic datasets}, CLUB always outperforms its uninformed competitors LinUCB-IND and
LinUCB-ONE, the gap getting larger as we either decrease the number of underlying clusters or we make the 
clusters sizes more and more unbalanced. Moreover, CLUB can clearly interpolate between these two competitors 
taking, in a sense, the best of both. On the other hand (and unsurprisingly), the informed competitors 
GOBLIN and CLEARVOYANT outperform all uninformed ones.
On the \textbf{``few hits" } scenario of LastFM, CLUB is again outperforming
both of its competitors. However, this is not happening in the \textbf{``many niches"} case
delivered by the Delicious dataset, where CLUB is clearly outperformed by LinUCB-IND.
The proposed alternative of CLUB that starts from an empty graph (Remark \ref{r:emptygraph})
might be an effective alternative in this case.
On the \textbf{Yahoo datasets} we extracted, CLUB tends to outperform its competitors, when
measured by CTR curves, thereby showing that clustering users solely based on past behavior can be beneficial.
In general, CLUB seems to benefit from situations where it is not immediately
clear which is the winner between the two extreme solutions (Lin)UCB-ONE and (Lin)UCB-IND,
and an adaptive interpolation between these two is needed.


\begin{figure}[t!]
\begin{picture}(5,30)(5,30)
\begin{tabular}{l@{\hspace{0.25pc}}l}
\includegraphics[width=0.234\textwidth]{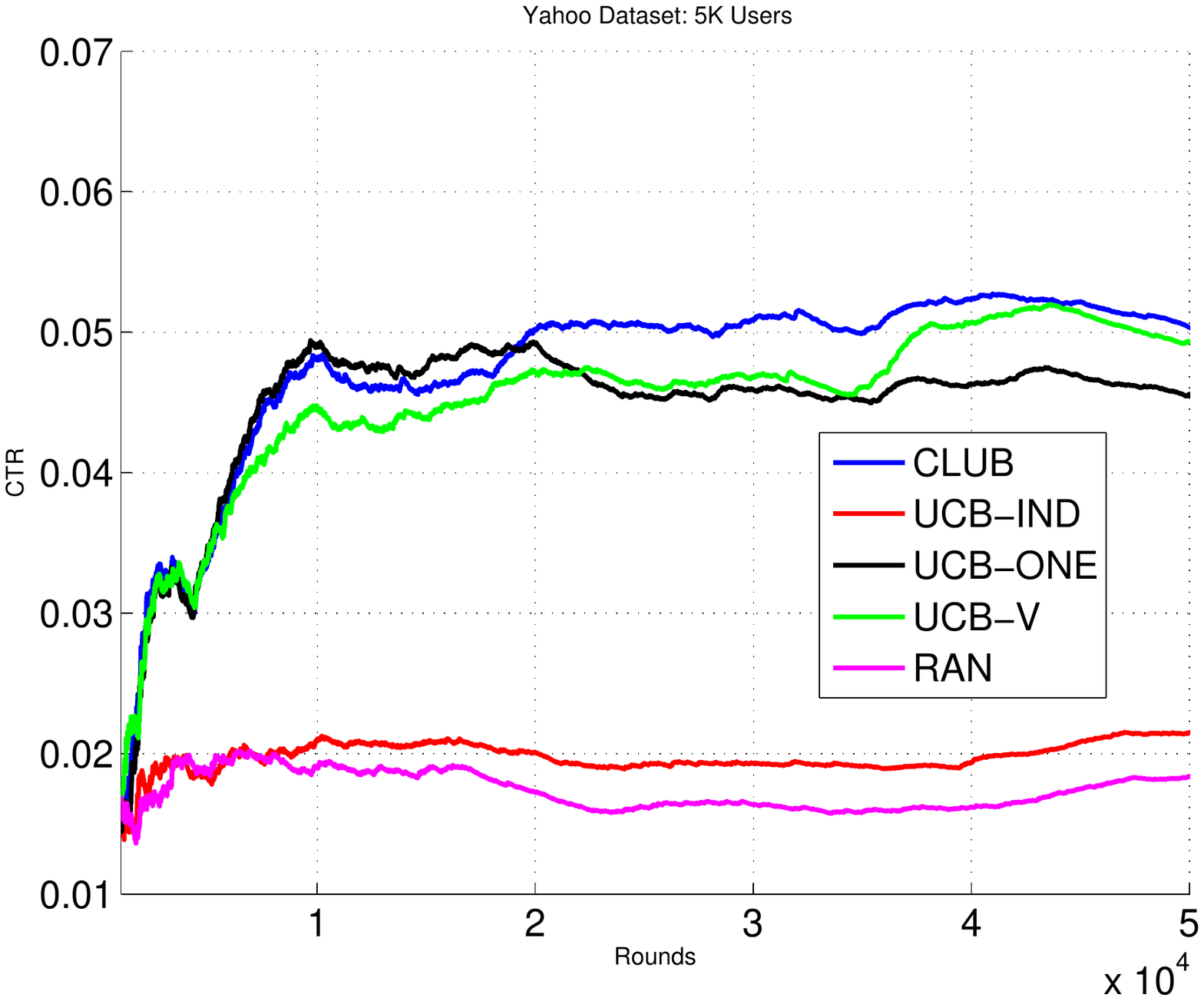}
& \includegraphics[width=0.234\textwidth]{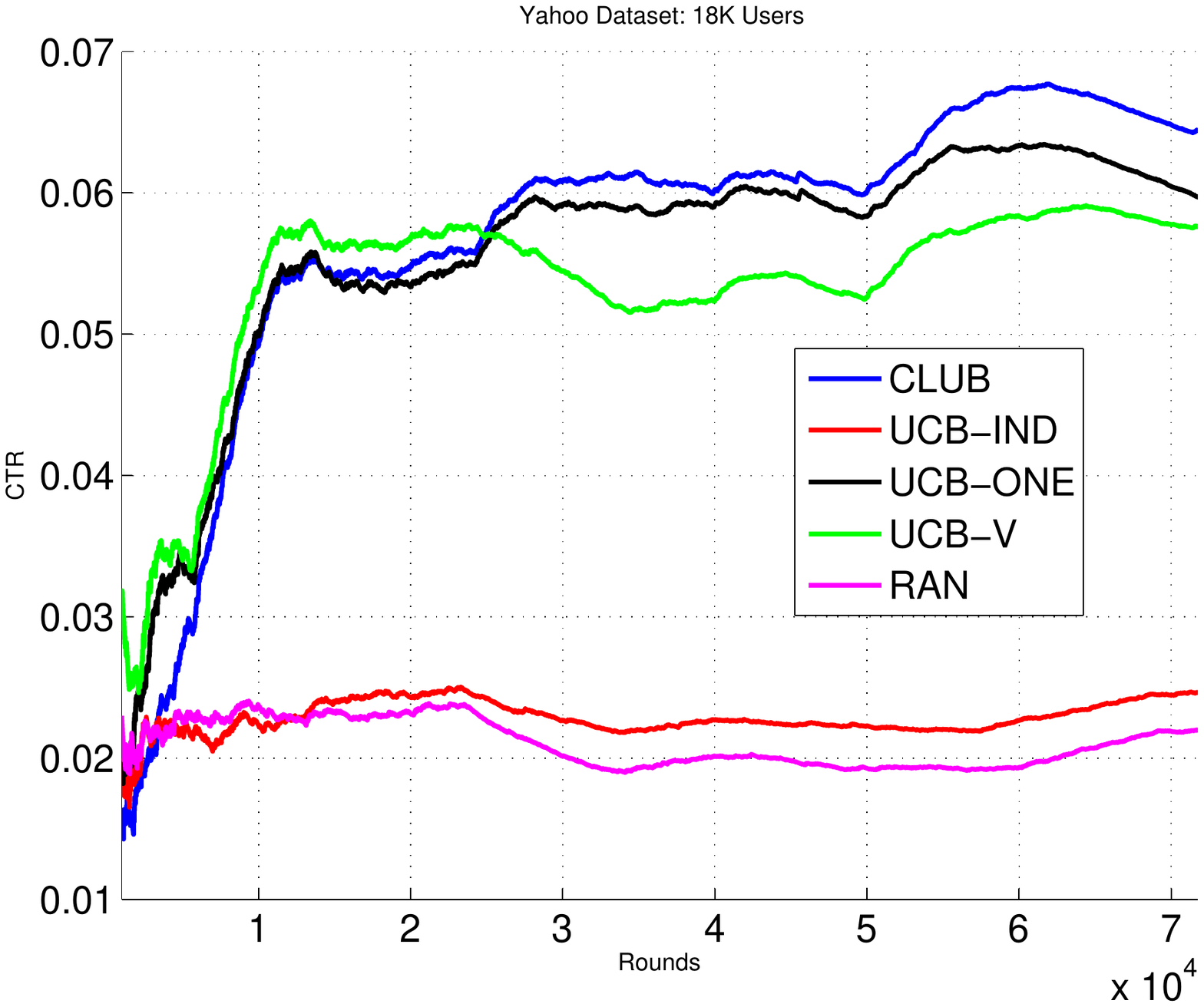}
\end{tabular}
\end{picture}
\vspace{0.93in}
\caption{\label{fig:yahoo}Plots on the Yahoo datasets
reporting Clickthrough Rate (CTR) over time, i.e., the fraction of times
the algorithm gets payoff one out of the number of retained records so far.}
\vspace{-0.1in}
\end{figure}

\section*{Acknowledgments}
We would like to thank the anonymous reviewers for their helpful and 
constructive comments. Also, the first and the second author acknowledge the 
support from Amazon AWS Award in Education Machine Learning Research Grant.
\newpage
\bibliographystyle{icml2014}
\bibliography{biblio}

\newpage

\appendix

\section{Appendix}
This supplementary material contains all proofs and technical details omitted from the main text,
along with ancillary comments, discussion about related work, and extra experimental results.

\subsection{Proof of Theorem 1}
%
%
The following sequence of lemmas are of preliminary importance. The first one
needs extra variance conditions on the process $X$ generating the context vectors.
%

We find it convenient to introduce the node counterpart to $\TCB_{j,t-1}(\bx)$, and
the cluster counterpart to $\sTCB_{i,t-1}$.
Given round $t$, node $i \in V$, and cluster index $j \in \{1, \ldots, m_t\}$, we let
\begin{align*}
\TCB_{i,t-1}(\bx) &=  \sqrt{\bx^\top M_{i,t-1}^{-1}\bx}\,\left(\sigma\sqrt{2\,\log \frac{|M_{i,t-1}|}{\delta/2}} + 1\right)\\
\sTCB_{j,t-1} &= \frac{\sigma\sqrt{2d\,\log t + 2\log (2/\delta)} + 1}
         {\sqrt{1+A_{\lambda}(\bT_{j,t-1}, \delta/(2^{m+1}d))}}~,
\end{align*}
being
\[
\bT_{j,t-1} = \sum_{i \in {\hat V_{j,t}}} T_{i,t-1} = |\{s \leq t-1\,:\, i_s \in {\hat V_{j,t}} \}|~,
\]
i.e.,
the number of past rounds where a node lying in cluster ${\hat V_{j,t}}$ was served. From a notational
standpoint, notice the difference\footnote
{
Also observe that $2nd$ has been replaced by $2^{m+1}d$ inside the log's.
}
between $\sTCB_{i,t-1}$ and $\TCB_{i,t-1}(\bx)$, both referring to a single node $i \in V$, and
$\sTCB_{j,t-1}$ and $\TCB_{j,t-1}(\bx)$ which refer to an aggregation (cluster) of nodes
$j$ among the available ones at time $t$.

\begin{lemma}\label{l:tcbupperbound}
Let, at each round $t$, context vectors $C_{i_t} = \{\bx_{t,1}, \ldots, \bx_{t,c_t}\}$ being generated
i.i.d. (conditioned on $i_t, c_t$ and all past indices $i_1, \ldots, i_{t-1}$,
rewards $a_1, \ldots, a_{t-1}$, and sets $C_{i_1}, \ldots, C_{i_{t-1}}$)
from a random process $X$ such that $||X|| = 1$, $\E[XX^\top]$ is full rank, with minimal eigenvalue
$\lambda > 0$. Let also, for any fixed unit vector $\bz \in \R^d$, the random variable
\(
(\bz^\top X)^2
\)
be (conditionally) sub-Gaussian with variance parameter\footnote
{
Random variable $(\bz^\top X)^2$ is conditionally sub-Gaussian with variance parameter $\sigma^2 > 0$ when
$\E_t\bigl[\exp(\gamma\,(\bz^\top X)^2)|\,c_t\,\bigr] \leq \exp\bigl(\sigma^2\,\gamma^2/2\bigr)$
for all $\gamma \in \R$.
The sub-Gaussian assumption can be removed here at the cost of assuming the conditional variance
of $(\bz^\top X)^2$ scales with $c_t$ like $\frac{\lambda^2}{c_t}$, instead of $\frac{\lambda^2}{\log (c_t)}$.
}
\[
\nu^2 = \Var_t\bigl[(\bz^\top X)^2\,|\,c_t\,\bigr]
\leq \frac{\lambda^2}{8\log (4c_t)}\quad \forall t~.
\]
Then
\[
\TCB_{i,t}(\bx) \leq \sTCB_{i,t}
\]
holds with probability at least $1-\delta/2$, uniformly over $i \in V$, $t = 0, 1, 2 \ldots $,
and $\bx \in \R^d$ such that $||\bx|| =1$.
\end{lemma}
\begin{proof}
Fix node $i \in V$ and round $t$.
By the very way the algorithm in Figure $1$ is defined, we have
\[
M_{i,t} = I + \sum_{s \leq t\,:\, i_s = i} {\bar \bx_s}{\bar \bx_s}^\top = I+ S_{i,t}~.
\]
First, notice that by standard arguments (e.g., \cite{dgs10}) we have
\[
\log |M_{i,t}| \leq d\,\log(1+T_{i,t}/d) \leq d\,\log(1+t)~.
\]
Moreover, denoting by $\lambda_{\max}(\cdot)$ and $\lambda_{\min}(\cdot)$ the maximal and the minimal eigenvalue
of the matrix at argument we have that, for any fixed unit norm $\bx \in \R^d$,
\[
\bx^\top M^{-1}_{i,t}\bx \leq \lambda_{\max}(M^{-1}_{i,t}) = \frac{1}{1+\lambda_{\min}(S_{i,t})}~.
\]
Hence, we want to show with probability at least $1-\delta/(2n)$ that
\begin{equation}
\label{e:boundgoal}
\begin{split}
\lambda_{\min}(S_{i,t}) \geq\
&\lambda T_{i,t}/4 - 8 \log \left( \frac{T_{i,t}+3}{\delta/(2nd)} \right) \\
&- 2 \sqrt{T_{i,t} \log\left(\frac{T_{i,t}+3}{\delta/(2nd)}\right)}
\end{split}
\end{equation}
holds for any fixed node $i$.
To this end, fix a unit norm vector $\bz \in \R^d$, a round $s \leq t$, and consider the variable
\begin{equation}
\begin{split}
V_s &= \bz^\top \left({\bar \bx_s}{\bar \bx_s}^\top - \E_s[{\bar \bx_s}{\bar \bx_s}^\top\,|\, c_s]\right)\bz \nonumber \\
&=
(\bz^\top {\bar \bx_s})^2 - \E_s[(\bz^\top {\bar \bx_s})^2\,|\, c_s]~.
\end{split}
\end{equation}
The sequence $V_1, V_2, \ldots, V_{T_{i,t}}$ is a martingale difference sequence, with optional skipping,
where $T_{i,t}$ is a stopping time.\footnote
{
More precisely, we are implicitly considering the sequence
$\eta_{i,1}V_1, \eta_{i,2}V_2, \ldots, \eta_{i,t}V_{t}$, where $\eta_{i,s} = 1$ if $i_s = i$, and
0 otherwise, with $T_{i,t} = \sum_{s=1}^t  \eta_{i,s}$.
}
Moreover, the following claim holds.
\begin{claim}\label{cl:1}
Under the assumption of this lemma, \\\\
\centerline{$\E_s[(\bz^\top {\bar \bx_s})^2\,|\, c_s] \geq \lambda/4$~.}
\end{claim}
{\em Proof of claim.} Let\footnote
{
This proof is based on standard arguments, and is reported here for the sake of completeness.
}
in round $s$ the context vectors be $C_{i_s} = \{\bx_{s,1}, \ldots, \bx_{s,c_s}\}$,
and consider the corresponding i.i.d. random variables
$Z_i = (\bz^\top\bx_{s,i})^2 - \E_s[(\bz^\top\bx_{s,i})^2\,|\,c_s],\ i = 1, \ldots, c_s$.
Since by assumption these variables are (zero-mean) sub-Gaussian, we have that (see, e.g., \cite{massart}[Ch.2])
\[
\Pr_s\left(Z_i < -a \,|\, c_t\right) \leq \Pr_s\left(|Z_i| > a \,|\, c_t\right) \leq 2e^{-a^2/2\nu^2}.
\]
holds for any $i$, where $\Pr_s(\cdot)$ is the shorthand for the conditional probability
\[
\Pr\left( \,\cdot\,\big|\, (i_1, C_{i_1}, a_1), \ldots, (i_{s-1}, C_{i_{s-1}}, a_{s-1}), i_s\, \right)~.
\]
The above implies
\begin{align*}
\Pr_s&\left( \min_{i = 1, \ldots, c_s} (\bz^\top\bx_{s,i})^2 \geq \lambda -a \,\Big|\, c_t\right) \\
&\qquad \qquad \qquad \qquad \geq \left(1-2e^{-a^2/2\nu^2}\right)^{c_s}~.
\end{align*}
Therefore
\begin{equation}
\begin{split}
\E_s[(\bz^\top {\bar \bx_s})^2\,|\, c_s]
&\geq
\E_s\left[  \min_{i = 1, \ldots, c_s}  (\bz^\top\bx_{s,i})^2\,\Big|\,c_s\right] \nonumber \\
&\geq
(\lambda-a)\, \left(1-2e^{-a^2/2\nu^2}\right)^{c_s}.
\end{split}
\end{equation}
Since this holds for all $a \in \R$, we set $a = \sqrt{2\nu^2\log(4c_s)}$ to get
$\left(1-2e^{-a^2/2\nu^2}\right)^{c_s} = (1-\frac{1}{2c_s})^{c_s} \geq 1/2$ (because $c_s \geq 1$),
and $\lambda - a \geq \lambda/2$ (because of the assumption on $\nu^2$). Putting together concludes
the proof of the claim. \qed

We are now in a position to apply a Freedman-like inequality for matrix martingales due to \cite{oliveira10, tropp11}
to the (matrix) martingale difference sequence
\[
\E_1[{\bar \bx_1}{\bar \bx_1}^\top\,|\, c_1] - {\bar \bx_1}{\bar \bx_1}^\top,\
\E_2[{\bar \bx_2}{\bar \bx_2}^\top\,|\, c_2] - {\bar \bx_2}{\bar \bx_2}^\top, \ldots ~
\]
with optional skipping.
Setting for brevity $X_s = {\bar \bx_s}{\bar \bx_s}^\top$, and
\[
W_t = \sum_{s \leq t\,:\, i_s = i} \left( \E_{s}[X^2_s\,|\, c_s] - \E^2_{s}[X_s\,|\, c_s]\right)~,
\]
Theorem 1.2 in \cite{tropp11} implies
\begin{align}
\label{e:tropp}
\Pr\Bigl( \exists t\,:\,\lambda_{\min} \left(S_{i,t} \right) \leq T_{i,t}\lambda_{\min}(\E_1[X_1\,|\,c_1]) &- a, ||W_t|| \leq \sigma^2 \Bigl) \notag \\
&\leq d\,e^{-\frac{a^2/2}{\sigma^2+2a/3}}~.
\end{align}
where $||W_t||$ denotes the operator norm of matrix $W_t$.

We apply Claim \ref{cl:1}, so that $\lambda_{\min}(\E_1[X_1\,|\,c_1]) \geq \lambda/4$,
and proceed as in, e.g., \cite{cbg08-ieee}. We set for brevity $A(x,\delta) = 2\log\frac{(x+1)(x+3)}{\delta}$,
and $f(A,r) = 2A + \sqrt{Ar}$. We can write
\begin{align*}
\Pr\Bigl(\exists t\,:\, \lambda_{\min} (S_{i,t}) \leq \lambda_{\min}T_{i,t}/4 - f(A(||W_t||,\delta),||W_t||)\Bigl)  \\
\leq \sum_{r=0}^{\infty} \Pr \Bigl( \exists t\,:\, \lambda_{\min} (S_{i,t}) \leq \lambda_{\min}T_{i,t}/4 - f(A(r,\delta),r),  \\
 \lfloor ||W_t||\rfloor = r\Bigl)  \\
\leq
\sum_{r=0}^{\infty}
\Pr\Bigl(\exists t\,:\, \lambda_{\min} \left(S_{i,t}\right) \leq \lambda_{\min}T_{i,t}/4 - f(A(r,\delta),r),  \\
||W_t||\leq  r + 1 \Bigl)\\
\leq d\,\sum_{r=0}^{\infty} e^{-\frac{f^2(A(r,\delta),r)/2}{r+1+2f(A(r,\delta),r)/3}}~,\qquad\qquad\qquad\qquad\qquad\quad
\end{align*}
the last inequality deriving from (\ref{e:tropp}). Because $f(A,r)$ satisfies
$f^2(A,r) \geq Ar + A+ \frac{2}{3}f(A,r)A$, we have that the exponent in the last exponential is at least $A(r,\delta)/2$,
implying
\[
\sum_{r=0}^{\infty} e^{-A(r,\delta)/2} = \sum_{r=0}^{\infty} \frac{\delta}{(r+1)(r+3)} < \delta
\]
which, in turn, yields
\begin{equation*}
\begin{split}
\Pr\Bigl(\exists t\,:\, \lambda_{\min} (S_{i,t}) &\leq T_{i,t}\lambda_{\min}/4 \\
 &\quad - f(A(||W_t||,\delta/d),||W_t||)\Bigl) \\
&\leq \delta~.
\end{split}
\end{equation*}
Finally, observe that
\begin{align*}
||W_t||
&\leq \sum_{s \leq t\,:\, i_s = i} ||\E_{s}[X^2_s\,|\, c_s]|| \\
&= \sum_{s \leq t\,:\, i_s = i} ||\E_{s}[X_s\,|\, c_s]|| \\
&\leq \sum_{s \leq t\,:\, i_s = i} \E_{s}[||X_s\,|\, c_s||] \\
&\leq T_{i,t}~.
\end{align*}
Therefore we conclude
\begin{align*}
\Pr\Bigl(\forall t\,:\,\lambda_{\min} (S_{i,t}) &\geq \lambda_{\min}T_{i,t}/4
- f(A(T_{i,t},\delta/d),T_{i,t})\Bigl)\\
 &\geq 1- \delta~.
\end{align*}
Stratifying over $i \in V$, replacing $\delta$ by $\delta/(2n)$ in the last inequality, and
overapproximating proves the lemma.
\end{proof}

\begin{lemma}\label{l:distanceupperbound}
Under the same assumptions as in Lemma \ref{l:tcbupperbound}, we have
\[
||\bu_i-\bw_{i,t}|| \leq \sTCB_{i,t}
\]
holds with probability at least $1-\delta$, uniformly over $i \in V$, and $t = 0, 1, 2, \ldots $.
\end{lemma}
%
\begin{proof}
From \cite{abbasi2011improved} it follows that
\[
|\bu_{i}^\top\bx -\bw_{i,t}^\top\bx| \leq \TCB_{i,t}(\bx)
\]
holds with probability at least $1-\delta/2$, uniformly over $i \in V$, $t = 0, 1, 2, \ldots$.
and $\bx \in \R^d$. Hence,
\begin{align*}
||\bu_i-\bw_{i,t}|| &\leq \max_{\bx \in \R^d\,:\, ||\bx|| = 1} |\bu_{i}^\top\bx -\bw_{i,t}^\top\bx| \\
&\leq \max_{\bx \in \R^d\,:\, ||\bx|| = 1} \TCB_{i,t}(\bx) \\
&\leq \sTCB_{i,t}~,
\end{align*}
the last inequality holding with probability $\geq 1-\delta/2$ by Lemma \ref{l:tcbupperbound}.
This concludes the proof.
\end{proof}

\begin{lemma}\label{l:edgedelete}
Under the same assumptions as in Lemma \ref{l:tcbupperbound}:
\begin{enumerate}
\item If $||\bu_i - \bu_j|| \geq \gamma$ and $\sTCB_{i,t} + \sTCB_{j,t} < \gamma/2$ then
\[
||\bw_{i,t}-\bw_{j,t}|| > \sTCB_{i,t} + \sTCB_{j,t}
\]
holds with probability at least $1-\delta$, uniformly over $i,j \in V$ and $t = 0, 1, 2, \ldots $;
\item if $||\bw_{i,t} - \bw_{j,t}|| > \sTCB_{i,t} + \sTCB_{j,t}$ then
\[
||\bu_{i}-\bu_{j}|| \geq \gamma
\]
holds with probability at least $1-\delta$, uniformly over $i,j \in V$ and $t = 0, 1, 2, \ldots $.
\end{enumerate}
\end{lemma}
%
\begin{proof}
\begin{enumerate}
\item We have
\begin{align}
\gamma
& \leq ||\bu_i - \bu_j|| \notag\\
& = ||\bu_i - \bw_{i,t} + \bw_{i,t} - \bw_{j,t} + \bw_{j,t} - \bu_j||\notag\\
& \leq ||\bu_i - \bw_{i,t}|| + ||\bw_{i,t} - \bw_{j,t}|| + ||\bw_{j,t} - \bu_j||\notag\\
& \leq \sTCB_{i,t} + ||\bw_{i,t} - \bw_{j,t}|| + \sTCB_{j,t}\notag\\
&\ \ \ \ \ {\mbox{(from Lemma \ref{l:distanceupperbound})}}\notag\\
& \leq ||\bw_{i,t} - \bw_{j,t}|| + \gamma/2, \notag
\end{align}
i.e.,
\(
||\bw_{i,t} - \bw_{j,t}|| \geq \gamma/2 > \sTCB_{i,t} + \sTCB_{j,t}~.
\)
\item Similarly, we have
\begin{align}
\sTCB_{i,t} + \sTCB_{j,t}
& < ||\bw_{i,t} - \bw_{j,t}|| \notag\\
& \leq ||\bu_i - \bw_{i,t}|| + ||\bu_{i} - \bu_{j}|| \notag \\
& \quad + ||\bw_{j,t} - \bu_j||\notag\\
& \leq \sTCB_{i,t} + ||\bu_{i} - \bu_{j}|| + \sTCB_{j,t}~,\notag
\end{align}
implying $||\bu_i - \bu_j|| > 0$. By the well-separatedness assumption, it must be the case
that $||\bu_i - \bu_j|| \geq \gamma$.
\end{enumerate}
\end{proof}

From Lemma \ref{l:edgedelete}, it follows that if any two nodes $i$ and $j$
belong to different true clusters
and the upper confidence bounds $\sTCB_{i,t}$ and  $\sTCB_{j,t}$ are both small enough, then
it is very likely that edge $(i,j)$ will get deleted by the algorithm (Lemma \ref{l:edgedelete}, Item 1).
Conversely, if the
algorithm deletes an edge $(i,j)$, then it is very likely that the two involved nodes $i$ and $j$
belong to different true clusters (Lemma \ref{l:edgedelete}, Item 2).
Notice that, we have $E \subseteq E_t$ with high
probability for all $t$.
Because the clusters  $\hat V_{1,t}, \ldots, \hat V_{m_t,t}$ are induced by the connected components of
$G_t = (V,E_t)$, every true cluster $V_i$ must be entirely included (with high probability) in some
cluster $\hat V_{j,t}$.
Said differently, for all rounds $t$,
the partition of $V$ produced by $V_1, \ldots, V_{m}$ is likely to be a refinement of the one
produced by $\hat V_{1,t}, \ldots, \hat V_{m_t,t}$ (in passing, this also shows that, with high probability,
$m_t \leq m$ for all $t$). This is a key property to all our analysis.
See Figure 2 in the main text for reference.

\begin{lemma}\label{l:tcbupperbound_cluster}
Under the same assumptions as in Lemma \ref{l:tcbupperbound}, if $\hj_t$ is the index of the
current cluster node $i_t$ belongs to, then we have
\[
\TCB_{\hj_t,t-1}(\bx) \leq \sTCB_{\hj_t,t-1}
\]
holds with probability at least $1-\delta/2$, uniformly over all rounds
$t = 1, 2, \ldots $, and $\bx \in \R^d$ such that $||\bx|| =1$.
\end{lemma}
%
\begin{proof}
The proof is the same as the one of Lemma \ref{l:tcbupperbound}, except that at the very end
we need to stratify over all possible shapes for cluster ${\hat V_{\hj_t,t}}$, rather than
over the $n$ nodes.
Now, since with
high probability (Lemma \ref{l:edgedelete}), ${\hat V_{\hj_t,t}}$ is the union of true clusters,
the set of all such unions is with the same probability upper bounded by $2^m$.
\end{proof}

The next lemma is a generalization
of Theorem 1 in \cite{abbasi2011improved}, and shows a convergence result for aggregate vector $\bbw_{j,t-1}$.
%
%
\begin{lemma}\label{l:comboapprox}
Let $t$ be any round, and assume the partition of $V$ produced by true clusters $V_1, \ldots, V_{m}$ is a
refinement of the one produced by the current clusters $\hat V_{1,t}, \ldots, \hat V_{m_t,t}$.
Let $j = \hj_t$ be the index of the current cluster node $i_t$ belongs to. Let
this cluster be the union of true clusters $V_{j_1}, V_{j_2}, \ldots, V_{j_k}$, associated with
(distinct) parameter vectors $\bu_{j_1}, \bu_{j_2}, \ldots, \bu_{j_k}$, respectively.
Define
\[
{\bar \bu_t} = \bM_{j,t-1}^{-1}
\left(\sum_{\ell = 1}^k \left(\frac{1}{k}\,I+\sum_{i \in V_{j_{\ell}}} (M_{i,t-1} -I)\right) \bu_{j_{\ell}} \right)~.
\]
Then:
\begin{enumerate}
\item Under the same assumptions as in Lemma \ref{l:tcbupperbound},
\[
||{\bar \bu_t} - \bbw_{j,t-1}|| \leq \sqrt{3m}\,\sTCB_{j,t-1}
\]
holds with probability at least $1-\delta$, uniformly over cluster indices $j = 1, \ldots, m_t$, and rounds
$t = 1, 2, \ldots $~.
\item For any fixed $\bu \in \R^d$ we have
\[
||{\bar \bu_t} - \bu|| \leq 2\,\sum_{\ell = 1}^k ||\bu_{j_{\ell}} - \bu|| \leq 2\,SD(\bu)~.
\]
\end{enumerate}
\end{lemma}
\begin{proof}
Let $X_{\ell,t-1}$ be the matrix whose columns are the $d$-dimensional vectors ${\bar \bx_s}$,
for all $s < t\,:\, i_s \in V_{j_{\ell}}$, $\ba_{\ell,t-1}$ be the column vector collecting all payoffs
$a_s$, $s < t\,:\, i_s \in V_{j_{\ell}}$, and $\boldeta_{\ell,t-1}$ be the corresponding column vector of noise values.
We have
\[
\bbw_{j,t-1} = \bM_{j,t-1}^{-1}\bbb_{j,t-1}~,
\]
with
\begin{align*}
\bbb_{j,t-1}
&= \sum_{\ell = 1}^k X_{\ell,t-1}\ba_{\ell,t-1} \\
&= \sum_{\ell = 1}^k X_{\ell,t-1} \left( X^\top_{\ell,t-1} \bu_{j_{\ell}} + \boldeta_{\ell,t-1 }\right)\\
&= \sum_{\ell = 1}^k \left( \sum_{i \in V_{j_{\ell}}} (M_{i,t-1} - I)\bu_{j_{\ell}}
+ X_{\ell,t-1}\,\boldeta_{\ell,t-1 }\right)~.
\end{align*}
Thus
\[
\bbw_{j,t-1} - {\bar \bu_t} = \bM_{j,t-1}^{-1}
\left( \sum_{\ell = 1}^k  \left(X_{\ell,t-1}\,\boldeta_{\ell,t-1} - \frac{1}{k}\,\bu_{j_{\ell}}\right) \right)
\]
and, for any fixed $\bx \in \R^d\,:\, ||\bx|| = 1$, we have
\begin{align*}
&\left(\bbw_{j,t-1}^\top\bx - {\bar \bu_t}^\top\bx \right)^2 \\
&= \left(\left( \sum_{\ell = 1}^k
\left(X_{\ell,t-1}\,\boldeta_{\ell,t-1}-\frac{1}{k}\,\bu_{j_{\ell}}\right)\right)^\top \bM_{j,t-1}^{-1} \bx \right)^2\\
&\leq
\bx^\top \bM_{j,t-1}^{-1} \bx\,
\left( \sum_{\ell = 1}^k  \left(X_{\ell,t-1}\,\boldeta_{\ell,t-1}-\frac{1}{k}\,\bu_{j_{\ell}}\right)\right)^\top
\bM_{j,t-1}^{-1} \\
& \quad \times \left( \sum_{\ell = 1}^k  \left(X_{\ell,t-1}\,\boldeta_{\ell,t-1}-\frac{1}{k}\,\bu_{j_{\ell}} \right)\right)\\
&\leq
2\,\bx^\top \bM_{j,t-1}^{-1} \bx\\
&\quad\times \Biggl(\Bigl( \sum_{\ell = 1}^k X_{\ell,t-1}\,\boldeta_{\ell,t-1} \Bigl)^\top
\bM_{j,t-1}^{-1} \Bigl( \sum_{\ell = 1}^k X_{\ell,t-1}\,\boldeta_{\ell,t-1} \Bigl)\\
& \qquad\qquad\qquad\qquad\qquad + \frac{1}{k^2}\,\Bigl(\sum_{\ell = 1}^k \bu_{j_{\ell}} \Bigl)^\top
\bM_{j,t-1}^{-1}
\Bigl(\sum_{\ell = 1}^k \bu_{j_{\ell}} \Bigl)\Biggl) \\
&\ \ \ \ {\mbox{(using $(a+b)^2 \leq 2a^2+2b^2$)~.}}
\end{align*}
We focus on the two terms inside the big braces.
Because $\hat V_{j,t}$ is made up of the union of true clusters, we can stratify
over the set of all such unions (which are at most $2^m$ with high probability),
and then apply the martingale
result in \cite{abbasi2011improved} (Theorem 1 therein), showing that
\begin{align*}
\left( \sum_{\ell = 1}^k X_{\ell,t-1}\,\boldeta_{\ell,t-1} \right)^\top
\bM_{j,t-1}^{-1}
\left( \sum_{\ell = 1}^k X_{\ell,t-1}\,\boldeta_{\ell,t-1} \right) \\
\leq
2\,\sigma^2\left(\log\frac{|\bM_{j,t-1}|}{\delta/2^{m+1}} \right)
\end{align*}
holds with probability at least $1-\delta/2$.
As for the second term, we simply write
\[
\frac{1}{k^2}\,\left( \sum_{\ell = 1}^k \bu_{j_{\ell}} \right)^\top
\bM_{j,t-1}^{-1}
\left( \sum_{\ell = 1}^k \bu_{j_{\ell}} \right)
\leq
\frac{1}{k^2}\,\Bigl|\Bigl|\sum_{\ell = 1}^k \bu_{j_{\ell}} \Bigl|\Bigl|^2
\leq 1~.
\]
Putting together and overapproximating we conclude that
\[
|\bbw_{j,t-1}^\top\bx - {\bar \bu_t}^\top\bx| \leq \sqrt{3m}\,\TCB_{j,t-1}(\bx)
\]
and, since this holds for all unit-norm $\bx$, Lemma \ref{l:tcbupperbound_cluster} yields
\[
||\bbw_{j,t-1} - {\bar \bu_t}|| \leq \sqrt{3m}\,\sTCB_{j,t-1}~,
\]
thereby concluding the proof of part 1.

As for part 2, because
\[
\bM_{j,t-1} = I + \sum_{\ell = 1}^k \sum_{i \in V_{j_{\ell}}} (M_{i,t-1}-I)~,
\]
we can rewrite $\bu$ as
\[
\bu = \bM_{j,t-1}^{-1}\left(\bu + \sum_{\ell = 1}^k \sum_{i \in V_{j_{\ell}}} (M_{i,t-1}-I)\bu\right)~,
\]
so that
\begin{align*}
{\bar \bu_t} - \bu
&= \bM_{j,t-1}^{-1}\Bigg(\frac{1}{k}\,\sum_{\ell = 1}^k (\bu_{j_{\ell}} - \bu) \\
&\quad +  \sum_{\ell = 1}^k \sum_{i \in V_{j_{\ell}}} (M_{i,t-1}-I)\,(\bu_{j_{\ell}} - \bu) \Bigg)~.
\end{align*}
Hence
\begin{align*}
||{\bar \bu_t} - \bu ||
&\leq
\frac{1}{k}\,\Bigl|\Bigl|\bM_{j,t-1}^{-1}\sum_{\ell = 1}^k (\bu_{j_{\ell}} - \bu)\Bigl|\Bigl|  \\
& \quad +\sum_{\ell = 1}^k \Bigl|\Bigl| \bM_{j,t-1}^{-1}\,\sum_{i \in V_{j_{\ell}}} (M_{i,t-1}-I)\,(\bu_{j_{\ell}} - \bu)\Bigl|\Bigl|\\
&\leq
\frac{1}{k}\,\sum_{\ell = 1}^k ||\bu_{j_{\ell}} - \bu)|| +
\sum_{\ell = 1}^k ||\bu_{j_{\ell}} - \bu||\\
&\leq
2\,\sum_{\ell = 1}^k ||\bu_{j_{\ell}} - \bu||~,
\end{align*}
as claimed.
\end{proof}

The next lemma gives sufficient conditions on $T_{i,t}$ (or on $\bT_{j,t}$) to insure that $\sTCB_{i,t}$
(or $\sTCB_{j,t}$) is small. We state the lemma for $\sTCB_{i,t}$, but the very same statement clearly
holds when we replace $\sTCB_{i,t}$ by $\sTCB_{j,t}$, $T_{i,t}$ by $\bT_{j,t}$, and $n$ by $2^m$.
\begin{lemma}\label{l:tcbsmallenough}
The following properties hold for upper confidence bound $\sTCB_{i,t}$:
\begin{enumerate}
\item $\sTCB_{i,t}$ is nonincreasing in $T_{i,t}$;
\item Let $A = \sigma\sqrt{2d\,\log (1+t) + 2\log (2/\delta)} + 1$.
Then
\[
\sTCB_{i,t} \leq \frac{A}{\sqrt{1+\lambda T_{i,t}/8}}
\]
when
\[
T_{i,t} \geq
                    \frac{2\cdot 32^2}{\lambda^2}\,\log\left(\frac{2nd}{\delta}\right)\,\log\left(\frac{32^2}{\lambda^2}
                                                                      \log\left(\frac{2nd}{\delta}\right)\right)~;
\]
\item We have
\[
\sTCB_{i,t} \leq  \gamma/4
\]
when
\begin{align*}
T_{i,t} \geq &\frac{32}{\lambda}\,\max\Biggl\{\frac{A^2}{\gamma^2},
\frac{64}{\lambda}\,\log\left(\frac{2nd}{\delta}\right)\,\\
&\qquad\qquad\qquad \times \log\left(\frac{32^2}{\lambda^2}\log\left(\frac{2nd}{\delta}\right)\right) \Biggl\}~.
\end{align*}
\end{enumerate}
\end{lemma}
\begin{proof}
The proof follows from simple but annoying calculations, and is therefore omitted.
\end{proof}

We are now ready to combine all previous lemmas into the proof of
Theorem 1.

\begin{proof}
Let $t$ be a generic round,
$\hj_t$ be the index of the current cluster node $i_t$ belongs to,
and $j_t$ be the index of the {\em true} cluster $i_t$ belongs to.
Also, let us define the aggregate vector $\bbw_{j_t,t-1}$ as follows~:
\begin{align*}
\bbw_{j_t,t-1} &= \bM_{{j_t},t-1}^{-1}\bbb_{j_t,t-1}, \\
\bM_{{j_t},t-1} &= I+\sum_{i \in V_{j_t}} (M_{i,t-1}-I),\\
\bbb_{{j_t},t-1} &= \sum_{i \in V_{j_t}} \bb_{i,t-1}~.
\end{align*}
Assume Lemma \ref{l:edgedelete} holds, implying that the current cluster $\hat V_{\hj_t,t}$
is the (disjoint) union of true clusters, and define the aggregate vector ${\bar \bu_t}$ accordingly,
as in the statement of Lemma \ref{l:comboapprox}.
Notice that $\bbw_{j_t,t-1}$ is the true cluster counterpart to $\bbw_{\hj_t,t-1}$, that is,
$\bbw_{j_t,t-1} = \bbw_{\hj_t,t-1}$ if $V_{j_t} = \hat V_{\hj_t,t}$.
Also, observe that ${\bar \bu_t} = \bu_{i_t}$ when $V_{j_t} =  \hat V_{\hj_t,t}$.
Finally, set for brevity
\[
\bx^*_t = \argmax_{\bx \in C_{i_t}} \bu_{i_t}^\top \bx
\]
We can rewrite the time-$t$ regret $r_t$ as follows:
\begin{align*}
r_t
&= \bu_{i_t}^\top\bx^*_t    - \bu_{i_t}^\top{\bar \bx_t}\\
&= \bu_{i_t}^\top\bx^*_t    - \bbw_{j_t,t-1}^\top\bx^*_t
+ \bbw_{j_t,t-1}^\top\bx^*_t    - \bbw_{\hj_t,t-1}^\top\bx^*_t\\
& \quad + \bbw_{\hj_t,t-1}^\top\bx^*_t  - \bbw_{j_t,t-1}^\top{\bar \bx_t}
+ \bbw_{j_t,t-1}^\top{\bar \bx_t} - \bu_{i_t}^\top{\bar \bx_t}~.
\end{align*}
Combined with
\[
\bbw_{\hj_t,t-1}^\top\bx^*_t + \TCB_{\hj_t,t-1}(\bx^*_t) \leq \bbw_{\hj_t,t-1}^\top{\bar \bx_t} + \TCB_{\hj_t,t-1}({\bar \bx_t}),
\]
and rearranging gives
\begin{align}
r_t
&\leq    \bu_{i_t}^\top\bx^*_t - \bbw_{j_t,t-1}^\top\bx^*_t - \TCB_{\hj_t,t-1}(\bx^*_t) \label{e:regr1}\\
& \quad + \bbw_{j_t,t-1}^\top{\bar \bx_t} - \bu_{i_t}^\top{\bar \bx_t} + \TCB_{\hj_t,t-1}({\bar \bx_t}) \label{e:regr2}\\
& \quad + (\bbw_{j_t,t-1}-\bbw_{\hj_t,t-1})^\top(\bx^*_t -{\bar \bx_t})~.\label{e:regr3}
\end{align}
%
We continue by bounding with high probability the three terms (\ref{e:regr1}),  (\ref{e:regr2}), and (\ref{e:regr3}).

As for (\ref{e:regr1}), and (\ref{e:regr2}), we simply observe that Lemma \ref{l:distanceupperbound} allows\footnote
{
This lemma applies here since, by definition, $\bbw_{j_t,t-1}$ is built only from payoffs from nodes in
$V_{j_t}$, sharing the common unknown vector $\bu_{i_t}$.
}
us to write
\[
\bu_{i_t}^\top\bx^*_t - \bbw_{j_t,t-1}^\top\bx^*_t \leq ||\bu_{i_t} - \bbw_{j_t,t-1}||
\leq \sTCB_{j_t,t-1}~,
\]
and
\[
\bbw_{j_t,t-1}^\top{\bar \bx_t} - \bu_{i_t}^\top{\bar \bx_t} \leq ||\bu_{i_t} - \bbw_{j_t,t-1}||
\leq \sTCB_{j_t,t-1}~.
\]
Moreover,
\begin{align*}
\TCB_{\hj_t,t-1}({\bar \bx_t})
&\leq \sTCB_{\hj_t,t-1}\\
&\ \ \ \ {\mbox{(by Lemma \ref{l:tcbupperbound_cluster})}}\\
&\leq \sTCB_{j_t,t-1}\\
&\ \ \ \ {\mbox{(by Lemma \ref{l:edgedelete} and the definition of $\hj_t$).}}
\end{align*}
Hence,
\begin{equation}\label{e:regr12bound}
(\ref{e:regr1}) + (\ref{e:regr2})\leq 3\sTCB_{j_t,t-1} 
\end{equation}
holds with probability at least $1-2\delta$, uniformly over $t$.

As for (\ref{e:regr3}), letting $\{\cdot\}$ be the indicator function of the predicate at argument,
we can write
\begin{align*}
(&\bbw_{j_t,t-1}-\bbw_{\hj_t,t-1})^\top(\bx^*_t -{\bar \bx_t})\\
&= (\bbw_{j_t,t-1}-\bu_{i_t})^\top(\bx^*_t -{\bar \bx_t})
+ (\bu_{i_t}-{\bar \bu_t})^\top(\bx^*_t -{\bar \bx_t})\\
&\quad + ({\bar \bu_t}-\bbw_{\hj_t,t-1})^\top(\bx^*_t -{\bar \bx_t})\\
&\leq
2\,\sTCB_{j_t,t-1}
+ 2\,||\bu_{i_t}-{\bar \bu_t}||
+2\sqrt{3m}\,\sTCB_{\hj_t,t-1}\\
&\ \ \ \ {\mbox{(using Lemma \ref{l:distanceupperbound}, $||\bx^*_t - {\bar \bx_t}|| \leq 2$,
                                                                 and Lemma \ref{l:comboapprox}, part 1)}}\\
&=
2\,\sTCB_{j_t,t-1}
+ 2\,\{V_{j_t} \neq \hat V_{\hj_t,t}\}\,||\bu_{i_t}-{\bar \bu_t}||\\
&\qquad\qquad+2\sqrt{3m}\,\sTCB_{\hj_t,t-1}\\
&\leq
2(1+\sqrt{3m})\,\sTCB_{j_t,t-1}
+ 4\,\{V_{j_t} \neq \hat V_{\hj_t,t}\}\,SD(\bu_{i_t}) \\
&\ \ \ \ {\mbox{(by Lemma \ref{l:edgedelete}, and Lemma \ref{l:comboapprox}, part 2)~.}}
\end{align*}
Piecing together we have so far obtained
\begin{align}\label{e:rtpartial}
r_t &\leq (5+2\sqrt{3m})\,\sTCB_{j_t,t-1} \notag \\
&\quad + 4\,\{V_{j_t} \neq \hat V_{\hj_t,t}\}\,SD(\bu_{i_t})~.
\end{align}
We continue by bounding $\{V_{j_t} \neq \hat V_{\hj_t,t}\}$.
From Lemma \ref{l:edgedelete}, we clearly have
\begin{align}
\{&V_{j_t} \neq \hat V_{\hj_t,t}\} \notag \\
&\leq \{ \exists i \in V_{j_t}, \exists j \notin V_{j_t}\,:\, (i,j) \in E_t \}\notag\\
&\leq \Bigl\{ \exists i \in V_{j_t}, \exists j \notin V_{j_t}\,:\, \forall s < t
\bigl( (i_s \neq i) \notag \\
&\ \ \ \vee (i_s = i, ||\bw_{i,s-1} + \bw_{j,s-1}|| \leq \sTCB_{i,s-1} + \sTCB_{j,s-1}) \bigr) \Bigl\}\notag\\
&\leq \{ \exists i \in V_{j_t}\,:\, \forall s < t\,\, i_s \neq i \} \notag\\
&\ \ \ \ \ + \Bigl\{ \exists i \in V_{j_t}, \exists j \notin V_{j_t}\,:\,\notag\\
&\qquad\  \forall s < t\,\,||\bw_{i,s-1} + \bw_{j,s-1}|| \leq \sTCB_{i,s-1} + \sTCB_{j,s-1} \Bigl\}\notag\\
&\leq \{ \exists i \in V_{j_t}\,:\, \forall s < t\,\, i_s \neq i \} \notag\\
&\ \ \ \ \ +
\{ \exists i \in V_{j_t}, \exists j \notin V_{j_t}\,:\, \notag\\
&\qquad\qquad\qquad\forall s < t\,\, \sTCB_{i,s-1} + \sTCB_{j,s-1} \geq \gamma/2 \}\notag\\
&\leq \{ \exists i \in V_{j_t}\,:\, \forall s < t\,\, i_s \neq i \} \notag \\
& \ \quad+ \{ \exists i \in V\,:\, \forall s < t\,\,\sTCB_{i,s-1} \geq \gamma/4 \}~.\notag
\end{align}
At this point, we apply Lemma \ref{l:tcbsmallenough} to $\sTCB_{i,t}$ with
\begin{align*}
A^2 &= \left(\sigma\sqrt{2d\,\log (1+T) + 2\log (2/\delta)} + 1\right)^2 \\
&\leq 4\sigma^2(d\,\log (1+T) + \log (2/\delta)) + 2,
\end{align*}
and set for brevity
\begin{align*}
B &= \frac{32}{\lambda}\,\max\Biggl\{\frac{A^2}{\gamma^2}, \frac{64}{\lambda}\,\log\left(\frac{2nd}{\delta}\right)\\
  &\qquad\qquad\qquad\qquad \times\log\left(\frac{32^2}{\lambda^2}\log\left(\frac{2nd}{\delta}\right)\right) \Biggl\}~,\\
C &=
                    \frac{2\cdot 32^2}{\lambda^2}\,\log\left(\frac{2^{m+1}d}{\delta}\right)\,\log\left(\frac{32^2}{\lambda^2}
                                                                      \log\left(\frac{2^{m+1}d}{\delta}\right)\right)~.
\end{align*}
We can write
\begin{align*}
\{& \exists i \in V\,:\, \forall s < t\,\,\sTCB_{i,s-1} \geq \gamma/4 \} \\
&\leq
\{ \exists i \in V\,:\, \sTCB_{i,t-2} \geq \gamma/4 \} \\
&\leq
\{ \exists i \in V\,:\, T_{i,t-2} \leq B \}~.
\end{align*}
Moreover,
\begin{align*}
\{& \exists i \in V_{j_t}\,:\, \forall s < t\,\, i_s \neq i \} \\
&\leq
\{\exists i \in V_{j_t}\setminus \{i_t\}\,:\, T_{i,t-1} = 0 \}\\
&\leq
\{\exists i \in V\,:\, T_{i,t-1} = 0 \}~.
\end{align*}
That is,
\begin{align*}
\{ V_{j_t} \neq \hat V_{\hj_t,t} \} &\leq
\{ \exists i \in V\,:\, T_{i,t-2} \leq B \} \\
&\ \quad+ \{\exists i \in V\,:\, T_{i,t-1} = 0 \}~.
\end{align*}
Further, using again Lemma \ref{l:tcbsmallenough} (applied this time to $\sTCB_{j,t}$)
combined with the fact that $\sTCB_{j,t} \leq A$ for all $j$ and $t$,
we have
\[
\sTCB_{j_t,t-1} = A\,\{\bT_{j_t,t-1} < C\} +  \frac{A}{\sqrt{1+\lambda\,\bT_{j_t,t-1}/8}}~,
\]
where
\[
\bT_{j_t,t-1} = \sum_{i \in V_{j_t}} T_{i,t-1} = |\{s \leq t-1\,:\, i_s \in V_{j_t} \}|~.
\]
Putting together as in (\ref{e:rtpartial}), and summing over $t = 1, \ldots, T$, we have shown
so far that with probability at least $1-7\delta/2$,
%
\begin{align*}
\sum_{t=1}^T r_t
&\leq (5+2\sqrt{3m})A\,\sum_{t=1}^T \{\bT_{j_t,t-1} < C\}\\
&\ \ \ \ \ + (5+2\sqrt{3m})A\,\sum_{t=1}^T\frac{1}{\sqrt{1+\lambda\,\bT_{j_t,t-1}/8}}\\
&\ \ \ \ \ + 4\,\sum_{t=1}^T SD(\bu_{i_t})\,\left\{\exists i \in V\,:\, T_{i,t-2} \leq B \right\} \\
&\ \ \ \ \ + 4\,\sum_{t=1}^T SD(\bu_{i_t})\,\{\exists i \in V\,:\,T_{i,t-1} = 0 \} ~,
\end{align*}
with $T_{i,t} = 0$ if $t \leq 0$.

We continue by upper bounding with high probability the four terms in the right-hand side of the last inequality.
First, observe that for any fixed $i$ and $t$, $T_{i,t}$ is a binomial random variable with parameters $t$ and $1/n$,
and $\bT_{j_t,t-1} = \sum_{i \in V_{j_t} } T_{i,t-1}$ which, for fixed $i_t$, is again binomial with parameters $t$,
and $\frac{v_{j_t}}{n}$, where $v_{j_t}$ is the size of the true cluster $i_t$ falls into.
Moreover, for any fixed $t$, the variables $T_{i,t}$, $i \in V$ are indepedent.

To bound the third term,
we use a standard Bernstein inequality twice: first, we apply it to sequences of independent
Bernoulli variables, whose sum $T_{i,t-2}$ has average $\E[T_{i,t-2}] = \frac{t-2}{n}$ (for $t \geq 3$), and then
to the sequence of variables $SD(\bu_{i_t})$ whose average $\E[SD(\bu_{i_t})] = \frac{1}{n}\,\sum_{i \in V} SD(\bu_{i})$
is over the random choice of $i_t$.

Setting for brevity
\[
D(B) = 2n\left(B + \frac{5}{3}\,\log(Tn/\delta)\right)+2,
\]
where $B$ has been defined before,
we can write
\begin{align*}
\sum_{t=1}^T &SD(\bu_{i_t})\,\left\{\exists i \in V\,:\, T_{i,t-2} \leq B \right\} \\
&=\sum_{t \leq D(B)} SD(\bu_{i_t})\,\left\{\exists i \in V\,:\, T_{i,t-2} \leq B \right\}\\
&\qquad + \sum_{t> D(B)} SD(\bu_{i_t})\,\left\{\exists i \in V\,:\, T_{i,t-2} \leq B \right\}\\
&\leq
\sum_{t \leq D(B)} SD(\bu_{i_t})\\
&\qquad + m\,\sum_{t> D(B)} \left\{\exists i \in V\,:\, T_{i,t-2} \leq B \right\}~.
\end{align*}
Then from Bernstein's inequality,
\[
\Pr\left( \exists i \in V\, \exists t > D(B)\,:\, T_{i,t-2} \leq B \right) \leq \delta~,
\]
and
\begin{align*}
\Pr\Biggl(\sum_{t \leq D(B)} SD(\bu_{i_t}) \geq \frac{3}{2}\,&D(B)\,\E[SD(\bu_{i_t})]\\
&\qquad + \frac{5}{3}\,m\,\log(1/\delta)  \Biggl) \leq \delta~.
\end{align*}
Thus with probability $\geq 1-2\delta$
\begin{align*}
\sum_{t=1}^T SD(\bu_{i_t})& \left\{\exists i \in V\,:\, T_{i,t-2} \leq B \right\} \\
&\leq
\frac{3}{2}\,D(B)\,\E[SD(\bu_{i_t})] + \frac{5}{3}\,m\,\log(1/\delta)~.
\end{align*}
Similarly, to bound the fourth term we have, with probability $\geq 1-2\delta$,
\begin{align*}
\sum_{t=1}^T SD(\bu_{i_t})&\left\{\exists i \in V\,:\,T_{i,t-1} =0 \right\}\\
&\leq \frac{3}{2}\,D(0)\,\E[SD(\bu_{i_t})] + \frac{5}{3}\,m\,\log(1/\delta) ~.
\end{align*}
Next, we crudely upper bound the first term as
\begin{align*}
(5+&2\sqrt{3m})A\,\sum_{t=1}^T \{\bT_{j_t,t-1} < C\} \\
&\leq (5+2\sqrt{3m})A\,\sum_{t=1}^T \{T_{i_t,t-1} < C\}~,
\end{align*}
and then apply a very similar argument as before to show that with probability $\geq 1-\delta$,
\[
\sum_{t=1}^T \{T_{i_t,t-1} < C\} \leq n\left(C + \frac{5}{3}\log\left(\frac{T}{\delta}\right)\right)+1~.
\]

Finally, we are left to bound the second term.
The following is a simple property of binomial random variables we be useful.
\begin{claim}\label{cl:binom}
Let $X$ be a binomial random variable with parameters $n$ and $p$, and $\lambda \in (0,1)$ be a constant.
Then
\[
\E \left[\frac{1}{\sqrt{1+\lambda X}} \right] \leq
\begin{cases}
\frac{3}{\sqrt{1+\lambda\,n\,p}} & {\mbox{if $np \geq 10~;$}}\\
1                                & {\mbox{if $np < 10$~.}}
\end{cases}
\]
\end{claim}
{\em Proof of claim.} The second branch of the inequality is clearly trivial, so we focus on the first one
under the assumption $np \geq 10$.
Let then $\beta \in (0,1)$ be a parameter that will be set later on. We have
\begin{align*}
\E \left[\frac{1}{\sqrt{1+\lambda X}} \right]
&\leq \Pr(X \leq (1-\beta)\,n\,p) \\
&\quad+ \frac{1}{\sqrt{1+\lambda\,(1-\beta)\,n\,p}}\,\Pr(X \geq (1-\beta)\,n\,p)\\
&\leq e^{-\beta^2\,n\,p/2} + \frac{1}{\sqrt{1+\lambda\,(1-\beta)\,n\,p}}~,
\end{align*}
the last inequality following from the standard Chernoff bounds.
Setting $\beta = \sqrt{\frac{\log(1+\lambda\,n\,p)}{n\,p}}$ gives
\begin{align*}
\E \left[\frac{1}{\sqrt{1+\lambda X}} \right]
&\leq
\frac{1}{\sqrt{1+\lambda\,n\,p}} \\
&\quad+ \frac{1}{\sqrt{1+\lambda\,(np-\sqrt{np\log(1+\lambda np)})}}\\
&\leq \frac{1}{\sqrt{1+\lambda\,n\,p}} + \frac{1}{\sqrt{1+\lambda\,n\,p/2}}\\
&\ \ \ \ {\mbox{(using $np \geq 10$)}}\\
&\leq \frac{3}{\sqrt{1+\lambda\,n\,p}}~,
\end{align*}
i.e., the claimed inequality\hfill \qed\\

Now,
\[
\E_{t-1}\left[\frac{1}{\sqrt{1+\lambda\,\bT_{j_t,t-1}/8}}\right]
= \sum_{j = 1}^m \frac{v_j}{n}\,\frac{1}{\sqrt{1+\lambda\,\bT_{j,t-1}/8}}~,
\]
being $\bT_{j,t-1} = |\{s < t\,:\, i_s \in V_j\}|$ a binomial variable with parameters $t-1$ and $\frac{v_j}{n}$,
where $v_j = |V_j|$. By the standard Hoeffding-Azuma inequality
\begin{align*}
\sum_{t=1}^T\frac{1}{\sqrt{1+\lambda\,\bT_{j_t,t-1}/8}} &\leq \sum_{t=1}^T \sum_{j = 1}^m \frac{v_j}{n}\,\frac{1}{\sqrt{1+\lambda\,\bT_{j,t-1}/8}} \\
&\quad+ \sqrt{2T\,\log(1/\delta)}
\end{align*}
holds with probability at least $1-\delta$,
In turn, from Bernstein's inequality, we have
\[
\Pr\left(\exists t\,\exists j\,:\, \bT_{j,t-1} \leq \frac{t-1}{2n}\,v_j - \frac{5}{3}\,\log(Tm/\delta)\right) \leq \delta~.
\]
Therefore, with probability at least $1-2\delta$,
\begin{align*}
&\sum_{t=1}^T\frac{1}{\sqrt{1+\lambda\,\bT_{j_t,t-1}/8}} \\
&\leq
\sum_{t=1}^T \sum_{j = 1}^m \frac{v_j}{n}\,\frac{1}{\sqrt{1+\frac{\lambda}{8}\,\left(\frac{t-1}{2n}\,v_j
                                                        - \frac{5}{3}\,\log(Tm/\delta) \right)_+}}
\\ & \quad +\sqrt{2T\,\log(1/\delta)}\\
&\leq
\sum_{j = 1}^m \frac{v_j}{n}\,\left(4n\,\frac{5}{3}\,\log(Tm/\delta)+1
+ \sum_{t=1}^T\frac{1}{\sqrt{1+\frac{\lambda}{8}\,\frac{t-1}{4n}\,v_j}}\right)
\\ & \quad +\sqrt{2T\,\log(1/\delta)}\\
&=
4n\,\frac{5}{3}\,\log(Tm/\delta)+1
+ \sum_{j = 1}^m \frac{v_j}{n}\,\sum_{t=1}^T\frac{1}{\sqrt{1+\frac{\lambda}{8}\,\frac{t-1}{4n}\,v_j}}
\\ & \quad +\sqrt{2T\,\log(1/\delta)}~.
\end{align*}
If we set for brevity $r_j = \frac{\lambda}{8}\,\frac{v_j}{4n}$, $j = 1, \ldots, m$,
we have
\begin{align*}
\sum_{t=1}^T\frac{1}{\sqrt{1+\frac{\lambda}{8}\,\frac{t-1}{4n}\,v_j}}
&\leq
\int_{0}^{T} \frac{dx}{\sqrt{1+(x-1)r_j}}\\
&=
\frac{2}{r_j} \left(\sqrt{1+T\,r_j - r_j} - \sqrt{1-r_j} \right)\\
&\leq
2\,\sqrt{\frac{T}{r_j}}~,
\end{align*}
so that
\begin{align*}
\sum_{t=1}^T\frac{1}{\sqrt{1+\lambda\,\bT_{j_t,t-1}/8}}
&\leq
4n\,\frac{5}{3}\,\log(Tm/\delta)+1\\
&\quad +\sqrt{2T\,\log(1/\delta)}
+ 8\,\sum_{j = 1}^m \sqrt{\frac{2 T v_j}{\lambda n}}~.
\end{align*}

Finally, we put all pieces together. In order for all claims to hold simultaneously with probability at least
$1-\delta$, we need to replace $\delta$ throughout by $\delta/10.5$. Then we switch to a ${\tilde O}$-notation,
and overapproximate once more to conclude the proof.
\end{proof}

\subsection{Implementation}\label{sa:implementation}
As we said in the main text, in implementing the algorithm in Figure $1$,
the reader should keep in mind that it is reasonable to
expect $n$ (the number of users) to be quite large, $d$ (the number of features of each item) to be relatively small,
and $m$ (the number of true clusters) to be very small compared to $n$.
Then the algorithm can be implemented by storing a least-squares estimator $\bw_{i,t-1}$
at each node $i \in V$, an aggregate least squares estimator $\bbw_{\hj_t,t-1}$ for each current cluster
$\hj_t \in \{1,\ldots, m_t\}$, and an extra data-structure which is able to perform decremental dynamic connectivity.
Fast implementations of such data-structures are those studied by~\cite{Tho97,kkm13}
(see also the research thread referenced therein). In particular, in \cite{Tho97} (Theorem 1.1 therein)
it is shown that a randomized construction exists that maintains a spanning forerst which, given an initial
undirected graph $G_1 = (V,E_1)$, is able to perform edge deletions and answer connectivity queries of the form
``Is node $i$ connected to node $j$" in expected total time
$O\left(\min\{|V|^2, |E_1|\,\log |V|\} + \sqrt{|V|\,|E_1|}\,\log^{2.5} |V|\right)$ for $|E_1|$ deletions.
Connectivity queries and deletions can be
interleaved, the former being performed in {\em constant} time. Notice that when we start off from the full graph,
we have $|E_1| = O(|V|^2)$, so that the expected amortized time per query becomes constant. On the other hand,
if our initial graph has
$|E_1| = O(|V|\,\log |V|)$ edges, then the expected amortized time per query is $O(\log^2 |V|)$. This becomes
$O(\log^{2.5} |V|)$ if the initial graph has $|E_1| = O(|V|)$.
In addition, we maintain an $n$-dimensional vector $\ci$ containing, for each node $i \in V$,
the index $j$ of the current cluster $i$ belongs to.

With these data-structures handy, we can implement our algorithm as follows.
After receiving $i_t$, computing $j_t$ is $O(1)$ (just by accessing $\ci$). Then, computing $k_t$ can be done
in time $O(d^2)$ (matrix-vector multiplication, executed $c_t$ times, assuming $c_t$ is a constant).
Then the algorithm directly updates $\bb_{i_t,t-1}$ and  $\bbb_{\hj_t,t-1}$, as well
as the inverses of matrices $M_{i_t,t-1}$ and
$\bM_{\hj_t,t-1}$, which is again $O(d^2)$, using standard formulas for rank-one adjustment of inverse matrices.
In order to prepare the ground for the subsequent edge deletion phase,
it is convenient that the algorithm also stores at each node $i$ matrix $M_{i,t-1}$ (whose time-$t$ update is again $O(d^2)$).

Let $\delete(i,\ell)$ and $\connected(i,\ell)$ be the two operations delivered by the decremental dynamic connectivity
data-structure.
Edge deletion at time $t$ corresponds to cycling through all nodes $\ell$ such that $(i_t,\ell)$ is an existing edge.
The number of such edges is on average equal to the average degree of node $i_t$, which is $O\left(\frac{|E_1|}{n}\right)$,
where $|E_1|$ is the number of edges in the initial graph $G_1$.
Now, if $(i_t,\ell)$ has to be deleted (each the deletion test being $O(d)$),
then we invoke $\delete(i_t,\ell)$, and then $\connected(i_t,\ell)$. If $\connected(i_t,\ell)$ = ``no",
this means that the current cluster $\hat V_{j_t,t-1}$ has to split into two new clusters as a consequence of the
deletion of edge $(i_t,\ell)$.
The set of nodes contained in these two clusters correspond to the two sets
\begin{align*}
&\{ k \in V\,:\, \connected(i_t,k) = ``yes"\},\\
&\{ k \in V\,:\, \connected(\ell,k) = ``yes"\}`,
\end{align*}
whose expected amortized computation {\em per node} is $O(1)$ to $O(\log^{2.5}n)$
(depending on the density of the initial graph $G_1$). We modify the $\ci$ vector accordingly,
but also the aggregate least squares estimators. This is because $\bbw_{\hj_t,t-1}$
(represented through $\bM^{-1}_{\hj_t,t}$ and $\bbb_{\hj_t,t}$) has to be spread over the two newborn clusters.
This operation can be
performed by adding up all matrices $M_{i,t}$ and all $\bb_{i,t}$, over all $i$ belonging to each of the
two new clusters
(it is at this point that we need to access $M_{i,t}$ for each $i$),
and then inverting the resulting aggregate matrices. This operation takes $O(n\,d^2 + d^3)$.
However, as argued in the comments following Lemma \ref{l:edgedelete},
with high probability the number of current clusters $m_t$ can never exceed $m$, so that with the
same probability this operation is only
performed at most $m$ times throughout the learning process.
Hence in $T$ rounds we have an overall (expected) running time
\begin{align*}
O\Biggl(&T\,\left(d^2 + \frac{|E_1|}{n}\,d\right) + m\,(n\,d^2 + d^3) + |E_1|\notag\\
&+ \min\{n^2, |E_1|\,\log n\} + \sqrt{n\,|E_1|}\,\log^{2.5} n\Biggl)~.
\end{align*}
%
Notice that the above is $n\cdot{\mbox{poly}}(\log n)$,
if so is $|E_1|$. In addition, if $T$ is large compared to $n$ and $d$, the average running time per round
becomes $O(d^2 + d\cdot{\mbox{poly}}(\log n))$.

As for memory requirements, we need to store two $d\times d$ matrices and one $d$-dimensional
vector at each node, one $d\times d$ matrix and one $d$-dimensional vector for each current cluster, vector
$\ci$, and the data-structures allowing for fast deletion and connectivity tests. Overall, these data-structures
do not require more than $O(|E_1|)$ memory to be stored, so that this implementation
takes $O(n\,d^2 + m\,d^2 + |E_1|) = O(n\,d^2+|E_1|)$, where we again relied upon the $m_t \leq m$ condition.
Again, this is $n\cdot{\mbox{poly}}(\log n)$ if so is $|E_1|$.

\subsection{Further Plots}\label{sa:furtherplots}
This section contains a more thorough set of comparative plots on the synthetic datasets described in the main text.
See Figure \ref{fig:artificial0} and Figure \ref{fig:artificial2}.


\begin{figure}[t!]
\begin{picture}(0,22)(0,22)
\begin{tabular}{l@{\hspace{-1.2pc}\vspace{6.4pc}}l}
\includegraphics[width=0.23\textwidth]{plots/z0m2eta01.pdf}
& \includegraphics[width=0.23\textwidth]{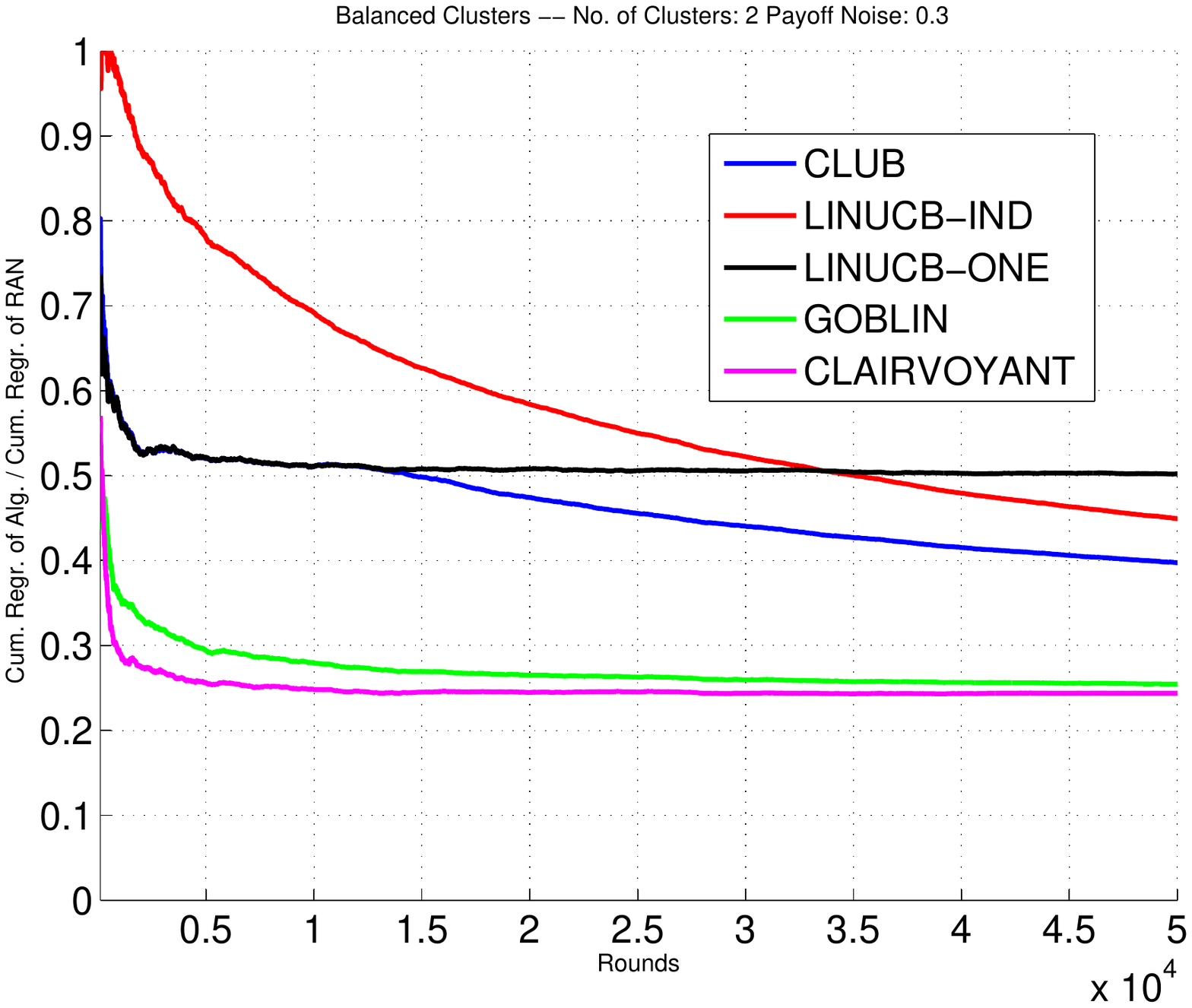}\\[-2.0in]
\includegraphics[width=0.23\textwidth]{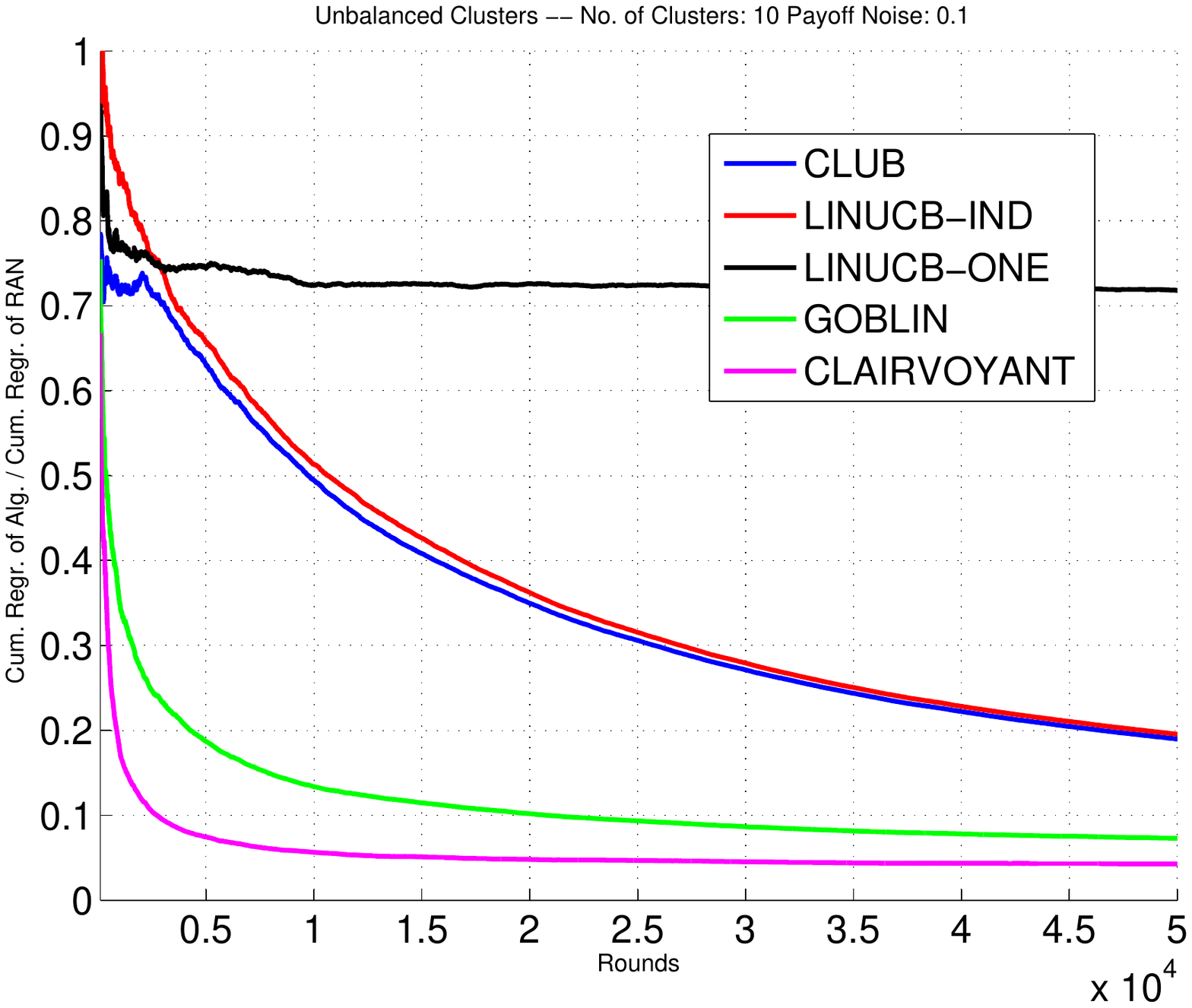}
& \includegraphics[width=0.23\textwidth]{plots/z0m10eta03.pdf}\\[-2.0in]
\end{tabular}
\end{picture}
\vspace{1.9in}
\caption{\label{fig:artificial0}Results on synthetic datasets. Each plot displays the behavior of
the ratio of the current cumulative regret of the algorithm (``Alg") to the current cumulative regret
of RAN, where
where ``Alg" is either ``CLUB" or ``LinUCB-IND" or ``LinUCB-ONE"
or ``GOBLIN''or ``CLAIRVOYANT''. The cluster sizes
are balanced ($z=0$). From left to right, payoff noise steps from $0.1$ to $0.3$, and from top to bottom
the number of clusters jumps from $2$ to $10$.}
\vspace{-0.15in}
\end{figure}

\begin{figure}[t!]
\begin{picture}(0,22)(0,22)
\begin{tabular}{l@{\hspace{0.22pc}\vspace{6.4pc}}l}
\includegraphics[width=0.23\textwidth]{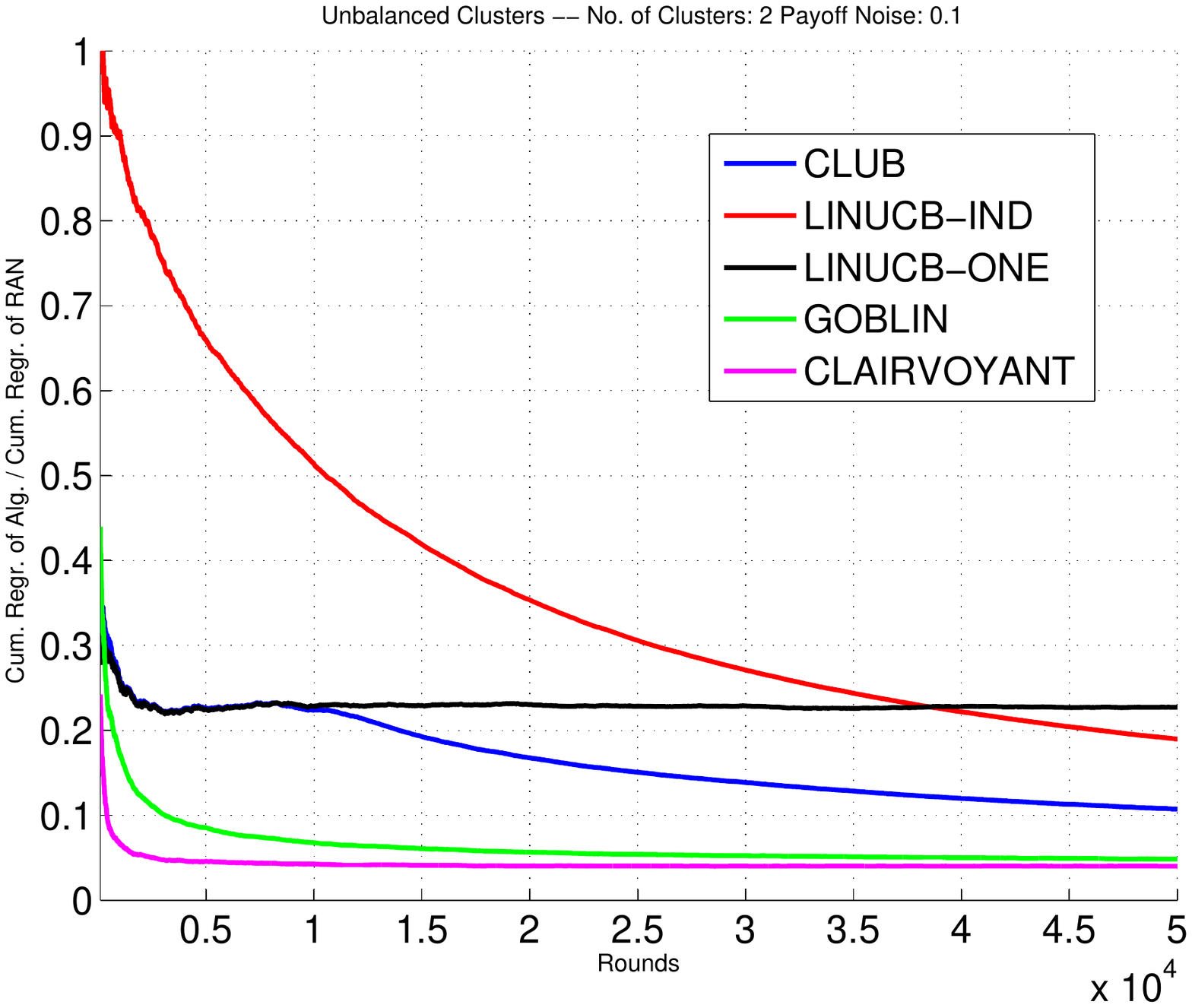}
& \includegraphics[width=0.23\textwidth]{plots/z2m2eta03.pdf}\\[-2.0in]
\includegraphics[width=0.23\textwidth]{plots/z2m10eta01.pdf}
& \includegraphics[width=0.23\textwidth]{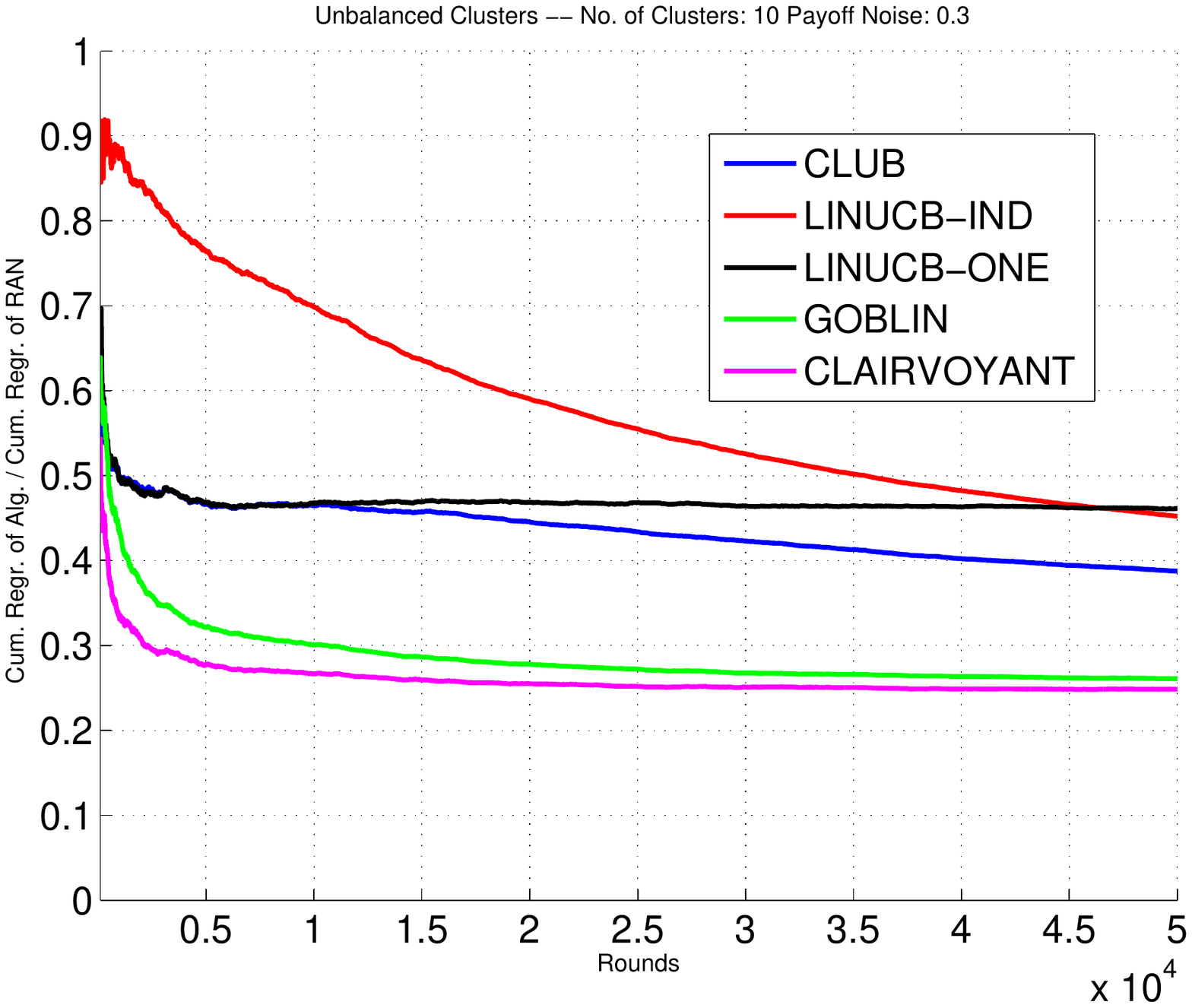}\\[-2.0in]
\end{tabular}
\end{picture}
\vspace{1.9in}
\caption{\label{fig:artificial2}Results on synthetic datasets in the case of unbalanced ($z=2$) cluster sizes.
The rest is the same as in Figure \ref{fig:artificial0}.}
\vspace{-0.15in}
\end{figure}

\subsection{Derivation of the Reference Bounds}
We now provide a proof sketch of the reference bounds mentioned in Section $2$ of the main text.

Let us start off from the {\em single user} bound for LINUCB (either ONE or IND) one can extract from 
\cite{abbasi2011improved}. Let $\bu_j \in \R^d$ be the profile vector of this user. Then, with probability
at least $1-\delta$, we have
\begin{align*}
\sum_{t=1}^T r_t 
&= 
O\left(\sqrt{T\left(\sigma^2\,d\,\log T + \sigma^2\,\log\frac{1}{\delta} +||\bu_i||^2\right)\,d\,\log T}\right)\\
&={\widetilde O}\left(\sqrt{T\left(\sigma^2\,d +||\bu_j||^2\right) d}\right)\\
&={\widetilde O}\left((\sigma\,d+\sqrt{d})\sqrt{T}\right),
\end{align*}
the last line following from assuming $||\bu_j|| = 1$.

Then, a straightforward way of turning this bound into a bound for the CLEARVOYANT algorithm that knows all clusters
$V_1, \ldots, V_m$ ahead of time and runs one instance of LINUCB per cluster is to sum the regret contributed
by each cluster throughout the $T$ rounds. Letting $T_{j,T}$ denote the set of rounds $t$ such that $i_t \in V_j$, we
can write
\[
\sum_{t=1}^T r_t ={\widetilde O}\left((\sigma\,d+\sqrt{d})\sum_{j=1}^m\sqrt{T_{j,T}}\right)~.
\]
However, because $i_t$ is drawn uniformly at random over $V$, we also have $\E[T_{j,T}] = T\frac{|V_j|}{n}$,
so that we essentially have with high probability
\[
\sum_{t=1}^T r_t = {\widetilde O}\left((\sigma\,d+\sqrt{d})\sqrt{T}\,\left(1+\sum_{j=1}^m\sqrt{\frac{|V_j|}{n}}\right)\right)~,
\]
i.e., Eq. ($1$) in the main text.

\subsection{Further Comments}\label{sa:fixeddesign}
As we said in Remark $3$,
a data-dependent variant of the CLUB algorithm can be designed and analyzed which relies on
data-dependent clusterability assumptions of the set of users with respect to a set of context vectors.
These data-dependent assumptions allow us to work in a fixed design setting for the sequence of context vectors
$\bx_{t,k}$, and remove the sub-Gaussian and full-rank hypotheses regarding $\E[XX^\top]$.
To make this more precise, consider an adversary that generates (unit norm) context vectors
in a (possibly adaptive) way that {\em for all} $\bx$ so generated\
\(
|\bu_j^\top\bx - \bu_{j'}^\top\bx| \geq \gamma~,
\)
\ whenever $j \neq j'$. In words, the adversary's power is restricted in that it cannot
generate two distict context vectors $\bx$ and $\bx'$ such that
$|\bu_j^\top\bx - \bu_{j'}^\top\bx|$ is small and $|\bu_j^\top\bx' - \bu_{j'}^\top\bx'|$ is large.
The two quantities must either be both zero (when $j = j'$) or both bounded away from 0
(when $j \neq j'$). Under this assumption, one can show that a modification to the $\TCB_{i,t}(\bx)$
and $\TCB_{j,t}(\bx)$ functions exists
that makes the CLUB algorithm
in Figure $1$ achieve a cumulative regret bound similar to the one in
(5), where the $\sqrt{\frac{1}{\lambda}}$ factor therein is turned back into
$\sqrt{d}$, as in the reference bound (1), but with a worse dependence
on the geometry of the set of $\bu_j$, as compared to $\E[SD(\bu_{i_t})]$.
The analysis goes along the very same lines as the one of Theorem $1$.

\subsection{Related Work}
The most closely related papers are \cite{dkc13,alb13,bl13,mm14}. 

In \cite{alb13}, the authors define a transfer learning problem within a stochastic
multiarmed bandit setting, where a prior distribution is defined over the set of possible models
over the tasks. More similar in spirit to our paper is the recent work \cite{bl13} that relies 
on clustering Markov Decision Processes based on their model parameter similarity.
A paper sharing significant similarities with ours, in terms of both setting and
technical tools is the very recent paper \cite{mm14} that came to our attention at the time 
of writing ours. In that paper, the authors analyze a noncontextual stochastic bandit problem where
model parameters can indeed be clustered in a few (unknown) types, thereby requiring the algorithm to 
learn the clusters rather than learning the parameters in isolation. Yet, the provided algorithmic 
solutions are completely different from ours.
Finally, in \cite{dkc13}, the authors work under the assumption that users are defined using a context 
vector, and try to learn a low-rank subspace under the assumption that variation across users is low-rank.
The paper combines low-rank matrix recovery with high-dimensional Gaussian Process Bandits, but it gives
rise to algorithms which do not seem easy to use in large scale practical scenarios.

\subsection{Ongoing Research}
%
%
This work could be extended along several directions.
First, we may rely on a softer notion of
clustering than the one we adopted here: a cluster is made up of nodes
where the ``within distance" between associated profile vectors is smaller than
their ``between distance". Yet, this is likely to require prior knowledge
of either the distance threshold or the number of underlying clusters,
which are assumed to be unknown in this paper. Second, it might be possible
to handle partially overlapping clusters. Third, CLUB can clearly be modified so
as to cluster nodes through off-the-shelf graph clustering techniques (mincut,
spectral clustering, etc.). Clustering via connected components has the twofold
advantage of being computationally faster and relatively easy to analyze.
In fact, we do not know how to analyze CLUB when combined with
alternative clustering techniques, and we suspect that Theorem $1$ already
delivers the sharpest results (as $T\rightarrow \infty$) when clustering
is indeed based on connected components only. Fourth, from a practical
standpoint, it would be important to incorporate further side information,
like must-link and cannot-link constraints. Fifth, in recommender systems practice,
it is often relevant to provide recommendations to new users, even in the absence
of past information (the so-called ``cold start" problem). In fact,
there is a way of tackling this problem through the machinery we developed here
(the idea is to duplicate the newcomer's node as many times as the current clusters
are, and then treat each copy as a separate user). This would potentially allow
CLUB to work even in the presence of (almost) idle users.
We haven't so far collected any experimental evidence on the effectiveness
of this strategy.
Sixth, following the comments we made in Remark $3$,
we are trying to see if the i.i.d. and other statistical assumptions we made
in Theorem $1$ could be removed.

\end{document}